\documentclass[10pt,twocolumn,letterpaper]{article}

\usepackage[pagenumbers]{cvpr} %

\input{preamble}
\usepackage[accsupp]{axessibility} %

\definecolor{cvprblue}{rgb}{0.21,0.49,0.74}
\usepackage[pagebackref,breaklinks,colorlinks,citecolor=cvprblue]{hyperref}

\def\paperID{x} %
\def\confName{CVPR}
\def\confYear{2024}

\title{Steerers: A framework for rotation equivariant keypoint descriptors}

\author{
      Georg Bökman$^\dag$\qquad Johan Edstedt$^\ddag$\qquad Michael Felsberg$^\ddag$\qquad Fredrik Kahl$^\dag$
     \\
     $^\dag$Chalmers University of Technology\qquad $^\ddag$Linköping University 
}

\begin{document}

\twocolumn[{%
\renewcommand\twocolumn[1][]{#1}%
\maketitle
\centering
\includegraphics[width=.497\textwidth]{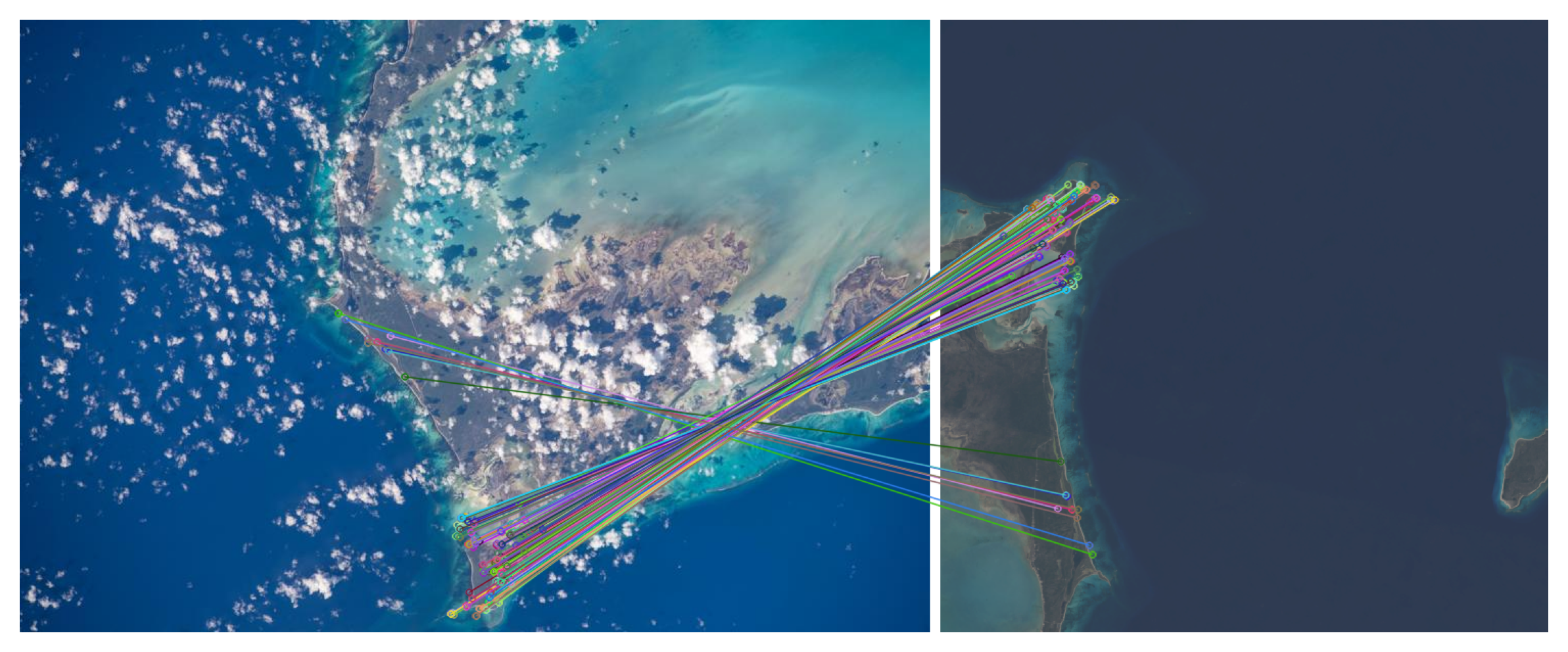}
\includegraphics[width=.497\textwidth]{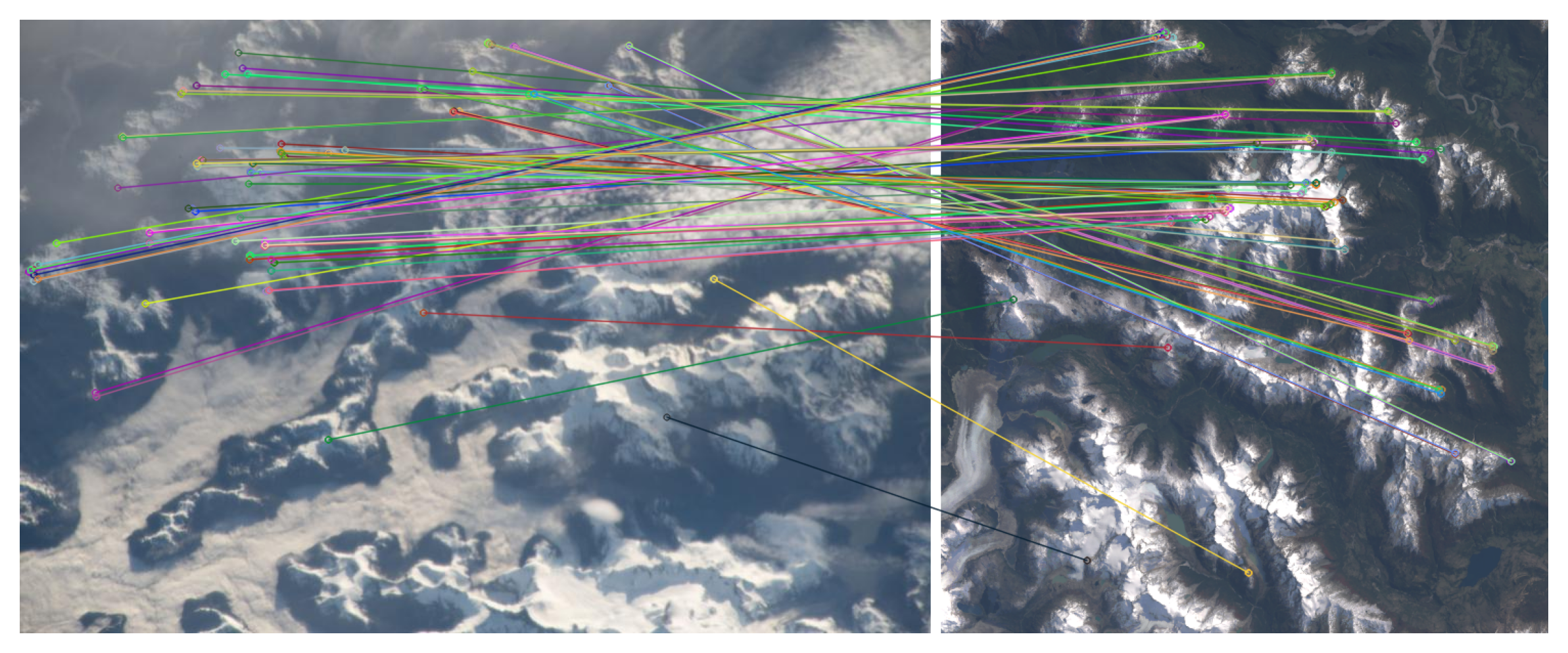} \vspace{-0.1cm}

\vspace{-1.5em}
\captionof{figure}{\textbf{Matching under large in-plane rotations.}
Two challenging pairs from AIMS~\cite{stoken2023astronaut}. 
The left images in each pair were taken by astronauts on the ISS and are geo-referenced by matching them with the satellite images on the right.
We plot estimated inlier correspondences after homography estimation with RANSAC.
Further qualitative examples are shown in the appendix.
\vspace{3em}
}
\label{fig:teaser}}]

\begin{abstract} 
Image keypoint descriptions that are discriminative and matchable over large changes in viewpoint are vital for 3D reconstruction. However, descriptions output by learned descriptors are typically not robust to camera rotation. 
While they can be made more robust by, e.g., data augmentation, this degrades performance on upright images. Another approach is test-time augmentation, which incurs a significant increase in runtime.
Instead, we learn a linear transform in description space that encodes rotations of the input image.
We call this linear transform a steerer since it allows us to transform the descriptions as if the image was rotated.
From representation theory, we know all possible steerers for the rotation group.
Steerers can be optimized (A) given a fixed descriptor, (B) jointly with a descriptor or (C) we can optimize a descriptor given a fixed steerer.
We perform experiments in these three settings
and obtain state-of-the-art results on the rotation
invariant image matching benchmarks AIMS and Roto-360.
We publish code and model weights at \href{https://github.com/georg-bn/rotation-steerers}{this https url}.
\end{abstract}    
\vspace{-0.2cm}
\section{Introduction}
\label{sec:intro}

Discriminative local descriptions are vital for multiple 3D vision tasks, and learned descriptors have recently been shown to outperform traditional handcrafted local features~\citep{detone2018superpoint,revaud2019r2d2,gleize2023silk,edstedt2024dedode}. 
One major weakness of learned descriptors compared to handcrafted features such as SIFT~\citep{lowe2004distinctive} is the relative lack of robustness to non-upright images~\cite{tyszkiewicz2020disk}. While images taken from ground level can sometimes be made upright by aligning with gravity as the canonical orientation, this is not always possible. For example, descriptors robust to rotation are vital in space applications~\cite{stoken2023astronaut},
as well as medical applications~\cite{pielawski2020comir}, where no such canonical orientation exists. Even when a canonical orientation exists, it may be difficult or impossible to estimate. Rotation invariant matching is thus a key challenge.

The most straightforward manner to get rotation invariant matching is to train or design a descriptor to be rotation invariant~\cite{lowe2004distinctive, detone2018superpoint}.
However, this sacrifices distinctiveness in matching
images with small relative rotations~\cite{pautrat2020lisrd}. An alternative approach is to train a rotation-sensitive descriptor and perform test-time-augmentation, %
selecting the pair that produces the most matches. The obvious downside of TTA is computational cost. For example, testing all $45^{\circ}$ rotations requires running the model eight times.

\begin{figure}
    \centering
    \includegraphics[width=0.9\linewidth]{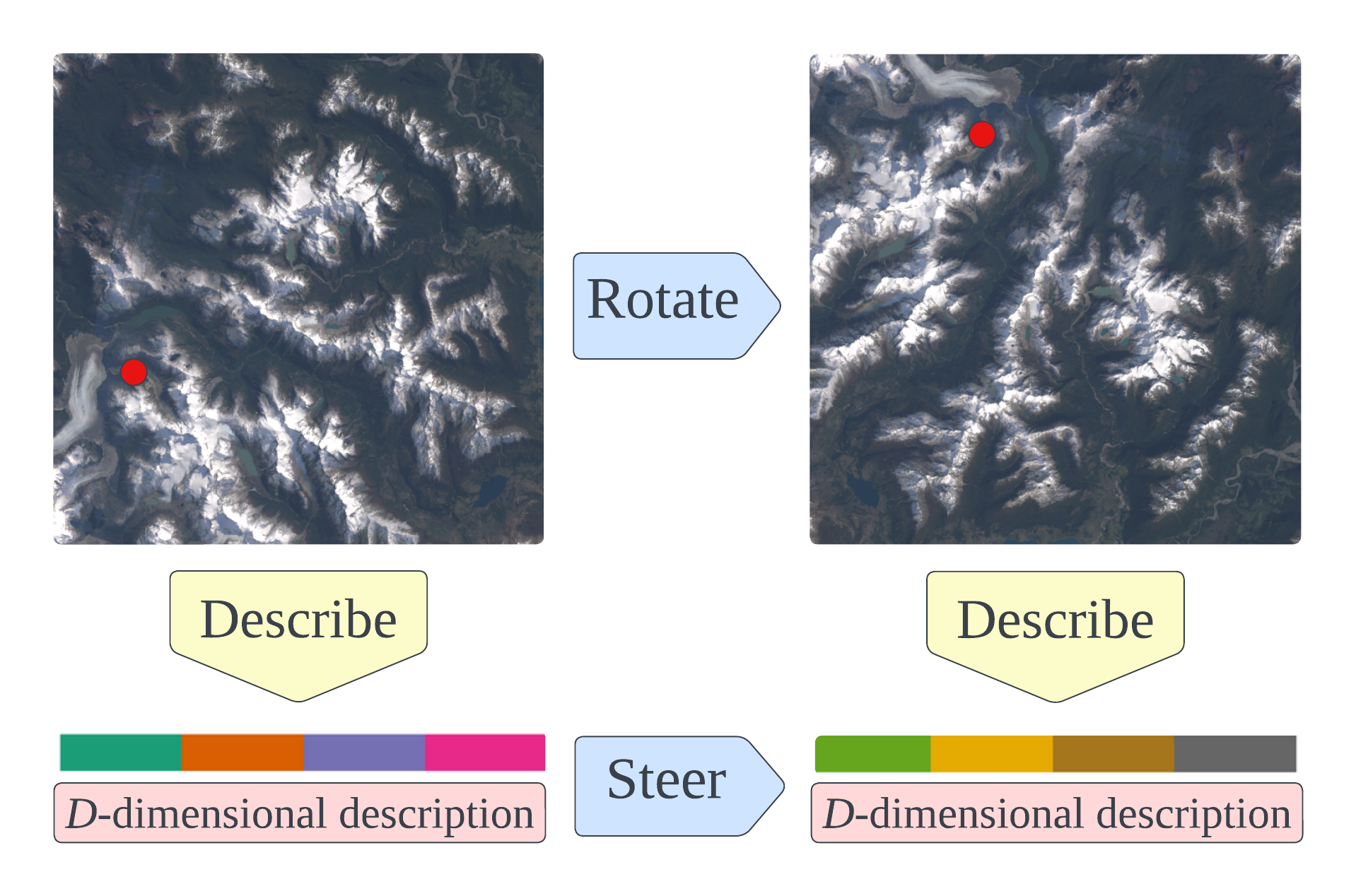} %
    \caption{\textbf{Overview of approach.} 
        A steerer (Definition~\ref{def:steerer}) is a linear map that transforms
        the description of a keypoint into the description
        of the corresponding keypoint in a rotated image.
        Thus, a steerer makes the keypoint descriptor rotation equivariant, and we can obtain the descriptions of keypoints in arbitrarily rotated images while only running the descriptor once.
    }
    \label{fig:approach}
\end{figure}

In this paper, we present an approach
that maintains distinctiveness for small rotations and 
allows for rotation invariant matching when we have images with large rotations. We do this while adding only negligible additional runtime, running the descriptor only a single time. The main idea is to learn a linear transform in description space that corresponds to a rotation of the input image; see Figure~\ref{fig:approach}.
We call this linear transform a \emph{steerer} as it allows us to modify keypoint descriptions as if they
were describing rotated images---we can \emph{steer} the descriptions without having to rerun the descriptor network.
We show empirically that
approximate steerers can be obtained for existing descriptors and motivate this theoretically.
We also investigate jointly optimizing steerers and descriptors and show how this enables nearly exact steering while not sacrificing performance on upright images.
Using mathematical representation theory,
we can describe all possible steerers---they are representations of the rotation group.
This enables the choice of a fixed steerer and training a descriptor for it, and in turn, the investigation of which steerers give the best performance.

Using our framework, we set a new state-of-the-art on the rotation invariant matching benchmarks AIMS~\cite{stoken2023astronaut} (Figure~\ref{fig:teaser}) and Roto-360~\cite{lee2023learning}. At the same time, we are with the same models able to perform on par with or even outperform existing non-invariant methods on upright images on the competitive MegaDepth-1500 benchmark~\cite{li2018megadepth,sun2021loftr}.

In summary, our main contributions are as follows.
\begin{enumerate}
    \item We introduce a new framework of steerers for equivariant keypoint descriptors (Section~\ref{sec:equiv_steer}) and theoretically motivate why steerers emerge in practice (Section~\ref{sec:desc_and_steer}).
    \item We develop several settings for investigating steerers (Section~\ref{sec:settings}) and ways to apply them for rotation invariant matching (Section~\ref{sec:matching_strat}).
    \item We conduct a large set of experiments, culminating in state-of-the-art on AIMS and Roto-360 (Section~\ref{sec:experiments}).
\end{enumerate}

\section{Related work}
\label{sec:related_work}

Classical keypoint descriptions are typically made rotation invariant by using keypoints with associated local rotation frames and computing the descriptions in these frames.
Examples include SIFT~\cite{lowe2004distinctive}, SURF~\cite{bay2008surf}, and ORB~\cite{rublee2011orb}.
A canonical rotation frame can be used for patch-based neural network descriptors as well~\cite{tian2019sosnet, tian2020hynet}.
Further, neural network-based approaches have been proposed for estimating the keypoint rotation frame~\cite{lee2021self, lee2022self, mishkin2018repeatability} and for both computing the rotation frame and the descriptions in that frame \cite{yi2016lift, lee2023learning}.
Notably, \cite{lee2022self, lee2023learning} use rotation equivariant ConvNets~\cite{cohen2016group, weiler2019general, Worall_2017_CVPR}.
Equivariant ConvNets have also been used
for rotation-robust keypoint detection without estimating the rotation frame~\cite{santellani2023s, bagad2022c}, keypoint description~\cite{bagad2022c, liu2019gift} and end-to-end image matching~\cite{bokman2022case}.
In theory\footnote{
It has been demonstrated that equivariant ConvNets can learn to break equivariance~\cite{edixhoven2023using} when this benefits the task at hand.
\Eg, the end-to-end matcher SE2-LoFTR~\cite{bokman2022case} is not perfectly consistent over rotations~\cite{bokman2022case,stoken2023astronaut}.
}, equivariant ConvNets guarantee that the predictions are consistent when rotating the image. They are one example of hard-coding equivariance into network layers using mathematical group theory, an idea that goes back to the 1990's~\cite{wood1996rep}
and has seen large recent interest~\cite{finzi2021practical,bronstein2021geometric,gerkenGeometricDeepLearning2023}.

Neural networks can also be encouraged to learn equivariance rather than having it hard-coded in the layers.
This can be done by enforcing group-specific invariants in the network output space~\cite{shakerinava2022structuring, gupta2023structuring} (see also Section~\ref{sec:desc_and_steer}).
Another approach is to specify a group representation on the output of the network and train the network to satisfy equivariance wrt.~that  representation~\cite{cohen2014transformation, Worrall_2017_ICCV, marchetti2023equivariant, koyama2023nft}.
We will use this approach for keypoint descriptions in our Setting C.
The benefits of not hard-coding equivariance are that arbitrary network architectures can be used (particularly pre-trained non-equivariant networks) and that one does not need to specify the group representations acting on each layer of the network.
A special case of learning equivariance is rotation invariant descriptors through data augmentation~\cite{ono2018lf, tian2017l2}.

A recent line of work~\cite{Lenc_2015_CVPR, gruver2023the, bokman2023investigating, Bruintjes_2023_CVPR} investigates to what extent neural networks exhibit equivariance without having been trained or hard-coded to do so.
They find that many networks are approximately equivariant.
One major limitation is that they only consider networks trained for image classification.
We will empirically demonstrate a high level of equivariance in keypoint descriptors that were not explicitly trained to be equivariant and theoretically motivate why this happens (our Setting A).

\section{Preliminaries}
In this work, we are interested in finding linear mappings between keypoint descriptions where the images may have been rotated independently.
We will, in particular, consider the group of quarter rotations $C_4$ and
the group of continuous rotations $\mathrm{SO}(2)$.

Ordinary typeset $g$ will denote an arbitrary group element,
boldface $\mathbf{g}$ will always mean the generator of $C_4$ for the remainder of the text so that the elements of $C_4$ are $\mathbf{g}$, $\mathbf{g}^2$, $\mathbf{g}^3$ and the identity element $\mathrm{id}=\mathbf{g}^4$.
Boldface $\mathbf{i}$ will denote the imaginary unit such that $\mathbf{i}^2=-1$.
Given matrices $X_1, X_2, \ldots, X_J$, the notation $\oplus_{j=1}^J X_j$ will mean
the block-diagonal matrix with blocks $X_1, X_2, \ldots, X_J$.

\subsection{Preliminaries on keypoint matching}
\label{sec:intro:matching}
The underlying task %
is to take two images of the same scene
and detect 2D points %
that correspond to the same 3D point.
A pair of such points that depict the same 3D point is called a correspondence.
The approach for finding correspondences that will be explored %
is a three-stage approach:
\begin{enumerate}
    \item Detection. Detect $N$ keypoint locations in each image.
    \item Description. Describe the keypoint locations with descriptors, \ie, feature vectors in $\mathbb{R}^D$.
    \item Matching. Match the descriptors, typically by using mutual nearest neighbours in cosine distance.
\end{enumerate}
This classical setup includes SIFT \cite{lowe2004distinctive} but also more recent deep learning-based approaches.
In particular, we follow the method in DeDoDe~\cite{edstedt2024dedode},
where the keypoint detector is first optimized to find good point tracks from SfM reconstructions and the keypoint descriptor is optimized by maximizing the matching likelihood %
obtained by a frozen keypoint detector as follows.
If the $N$ descriptors (each normalized to unit length) in the two images are $y_1\in\mathbb{R}^{D\times N}$ and
$y_2\in\mathbb{R}^{D\times N}$ 
we first form the $N\times N$ matching matrix $Y = y_1^T y_2$, and
obtain a matrix of 
pairwise likelihoods by using the dual softmax \cite{rocco2018neighbourhood, tyszkiewicz2020disk, sun2021loftr}:
\begin{equation} \label{eq:match_prob}
    p(y_1, y_2) =
    \frac{\exp (\iota Y)}{\sum_{\text{columns}}\exp (\iota Y)}
    \cdot \frac{\exp (\iota Y)} {\sum_{\text{rows}}\exp (\iota Y)}.
\end{equation}
Here $\iota=20$ is the inverse temperature.
The negative logarithm of the likelihood \eqref{eq:match_prob} is minimized
for those pairs in the $N\times N$ matrix that correspond to
ground truth inliers.

\subsection{Preliminaries on group representations}
\label{sec:intro:group}

\begin{definition}(Group representation)
Given a group $G$, a representation of $G$ on $\mathbb{R}^D$ is a mapping
$
    \rho: G \to \text{GL}(\mathbb{R},D)
$
that preserves the group multiplication, \ie, $\rho(g g') = \rho(g)\rho(g')$ for every $g, g'\in G$.
\end{definition}

Simply stated, $\rho$ maps every element in the group to an invertible $D\times D$ matrix.
The point of using representations is that groups such as $C_4$ act differently on different quantities
as we will illustrate in the following examples.

\begin{exmp}\label{ex:repimage}
    For $\mathbb{R}^{n\times n}$ (a square image grid), $C_4$ can be represented by permutations of the pixels in the obvious way so that the image is rotated anticlockwise by multiples of $\ang{90}$.
    We denote this group representation by
    $P_{90}$ so %
    that applying $P_{90}^k$ rotates
    the image by $k\cdot\ang{90}$ anticlockwise.
\end{exmp}

\begin{exmp}\label{ex:repcoord}
    For $\mathbb{R}^2$ (image coordinates), one possible representation of $C_4$ is 
    $
    \rho(\mathbf g^k) = R_{90}^k = \begin{psmallmatrix}
        0 & -1 \\
        1 & 0 \\
    \end{psmallmatrix}^k.
    $
    Multiplication by $R^k_{90}$ corresponds to rotating image coordinates by $k\cdot\ang{90}$ if the center of the image is taken as $(0, 0)$.
\end{exmp}

\begin{figure}
    \centering
    \includegraphics[width=\linewidth]{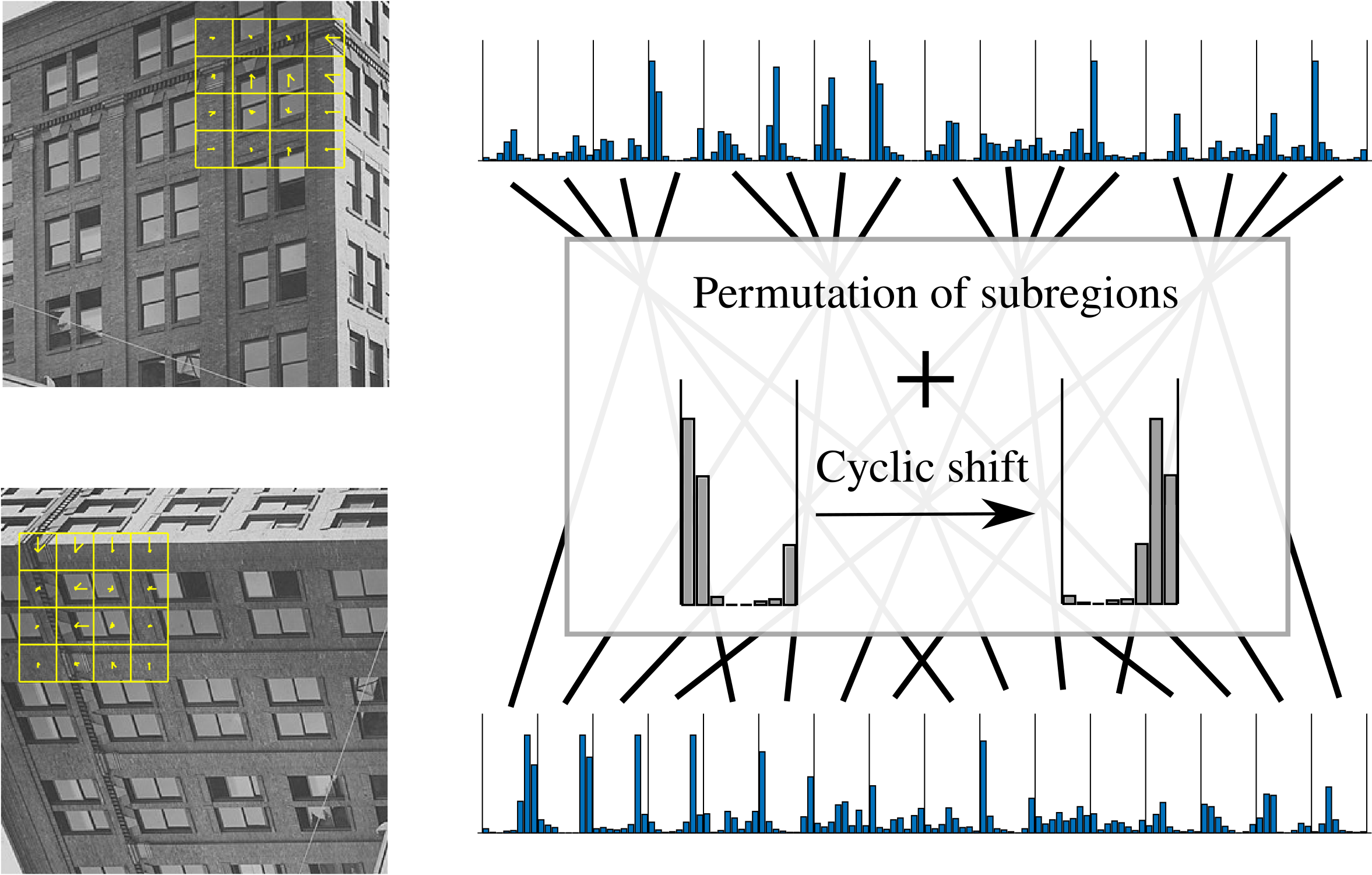}
    \caption{\textbf{Equivariance of Upright SIFT.} Left: A keypoint with its Upright SIFT description in an upright image and a rotated version. The small yellow squares are the subregions where histograms of gradient orientations are computed.
    Right: The Upright SIFT descriptions %
    unravelled into the 128 bin histograms that constitute them. When we rotate the image, the subregions 
    are permuted, and the histogram bins within each subregion are further permuted cyclically.
    Hence, Upright SIFT is rotation equivariant.
    }
    \label{fig:upsift}
\end{figure}

\section{Equivariance and steerability}
\label{sec:equiv_steer}
In this section, we will analyze the close connection between 
equivariance and steerability. We start with an example to introduce the former concept.

\begin{exmp}\label{ex:repsift}
    SIFT descriptions~\cite{lowe2004distinctive} are $128$ dim. vectors
    designed to be invariant to rotation, scale and illumination and highly distinctive for leveraging feature matching.
    For an input image $I\in\mathbb{R}^{n\times n}$ and $N$ keypoints with scale and orientation $x\in\mathbb{R}^{4\times N}$\footnote{The first two coordinates of each keypoint in $x$ are its location and the last two a vector for its orientation and scale, so $x$ is rotated by $\oplus_{b=1}^2 R_{90}$.},
    we get descriptions $y\in\mathbb{R}^{128\times N}$.
    If $f$ is the SIFT descriptor, we write $f(I, x) = y$.
    The descriptions consist of histograms of image gradients over patches around the keypoints $x$.
    The patches are oriented by the keypoint orientations so that
    the descriptions are invariant to joint rotations of the image and keypoints:
    \[
        f\left(P_{90}^k I, (\oplus_{b=1}^2 R_{90})^k x\right) = f(I, x).
    \]
    If we discard the keypoint orientations, \ie, set the angle of each keypoint to $0$, 
    we get the Upright SIFT (UPSIFT) descriptor~\cite{bay2006surf,baatz2010handling}, which is often used for upright images as it is more discriminative than SIFT.
    When we rotate an image $\ang{90}$, then the gradient histograms,
    \ie, the UPSIFT descriptions are permuted by a specific permutation $P_\text{UPSIFT}$,
    so if $f$ is the UPSIFT descripor, we have
    \[
        f\left(P_{90}^k I, (\oplus_{b=1}^2 R_{90})^k x\right) = P_\text{UPSIFT}^k f(I, x).
    \]
    We illustrate the permutation $P_\text{UPSIFT}$ in Figure~\ref{fig:upsift}.
    UPSIFT is not rotation invariant, but it is rotation \emph{equivariant}---when we rotate the input, the output changes predictably.
    Explicitly, the representation %
    is $\rho(\mathbf g^k)=P_\text{UPSIFT}^k$.
\end{exmp}

\begin{definition}[Equivariance]
We say that a function $f: V\to W$ is equivariant with respect to a group $G$ if
\begin{equation}
    \rho(g) f(v) = f(\rho_{\text{in}}(g) v), \forall v \in V, g\in G,
\end{equation}
for some group representations $\rho_{\text{in}}, \rho$.
\end{definition}

This work will mainly be concerned with the equivariance of learned keypoint descriptors of ordinary keypoints (without scale and orientation).

\begin{definition}[Equivariance of keypoint descriptor]
We say that a keypoint descriptor $f$ is equivariant with respect to a group $G$ transforming the input image by %
$\rho_{\text{image}}$
and the input keypoint locations by %
$\rho_\text{keypoint}$
if there exists %
$\rho$ such that
\begin{equation}
    \rho(g) f(I, x) = f(\rho_{\text{image}}(g) I, \rho_\text{keypoint}(g) x)
\end{equation}
for all images, keypoints and group elements.
We call the descriptor invariant if $\rho(g)$ is the identity matrix for all $g$. Invariance is a special type of equivariance.
\end{definition}

\begin{exmp}
A keypoint descriptor $f$ is equivariant under $\ang{90}$ rotations if there exists
$\rho$ of $C_4$ such that
\begin{equation} \label{eq:equiv_desc}
    \rho(\mathbf g^k) f(I, x) = f(P_{90}^k I,R_{90}^k x)
\end{equation}
for $k\in\{0,1,2,3\}$,
where $P_{90}$ and $R_{90}$ are the representations from Examples \ref{ex:repimage} and \ref{ex:repcoord} that rotate images and coordinates in the ordinary manner.

Both SIFT and Upright SIFT are equivariant.
For SIFT, $\rho(\mathbf g^k)$ is the identity, so SIFT is invariant.
For Upright SIFT, $\rho(\mathbf g^k)$ is $P_\text{UPSIFT}^k$ as explained in Example~\ref{ex:repsift}.
\end{exmp}

One aim of this work is to argue and demonstrate that learned keypoint descriptors, which are
trained on upright data, will behave more like Upright SIFT than SIFT, \ie, they will be rotation equivariant but not invariant.

\begin{definition}[Steerability, adapted from \cite{freeman1991design}]
\label{def:steerability}
A real-valued function $\phi:V\to \mathbb{R}$ is said to be steerable under a representation $\rho_\text{in}$ of $G$ on $V$,
if there exist $D$ functions (for some $D$) $\phi_j : V \to \mathbb{R}$
and $D$ functions $\kappa_j : G \to \mathbb{R}$ such that
$
    \phi(\rho_\text{in}(g) v) = 
    \sum_{j=1}^D \kappa_j(g) \phi_j(v).
$
\end{definition}

Note that an equivariant function $f:V\to \mathbb{R}^D$ satisfies in each component $f_d$ that 
$
    f_d(\rho_\text{in}(g)v) = \sum_{j=1}^D\rho(g)_{dj}f_j(v),
$
so each component of $f$ is steerable, in the notation of Definition~\ref{def:steerability}, $\phi=f_d, \phi_j=f_j, \kappa_j(g)=\rho(g)_{dj}$.
This motivates the definition of a \emph{steerer}. %

\begin{definition}[Steerer]
\label{def:steerer}
    Given a function $f:V\to W$ between vector spaces, and a representation $\rho_\text{in}$
    of $G$ on $V$,
    a steerer is a representation $\rho$ of $G$ on $W$ that makes $f$ equivariant, \ie such that
    \begin{equation}\label{eq:steerer}
        f(\rho_\text{in}(g) v) = \rho(g) f(v).
    \end{equation}
    Even if \eqref{eq:steerer} only holds approximately or $\rho$ is only approximately a representation, we will refer to $\rho$ as a steerer.
\end{definition}

We will use the verb \emph{steer} for multiplying a feature/description by a steerer; see Figure~\ref{fig:approach} for the broad idea.

\begin{exmp}
    As explained in Example~\ref{ex:repsift}, $P_\text{UPSIFT}$ is a steerer for Upright SIFT under $\ang{90}$ rotations.
    This has practical consequences.
    If we want to obtain the Upright SIFT descriptions for an image $I$ and the same image rotated $k\cdot\ang{90}$, we only need to compute the descriptions for the original image. We can obtain the rotated ones by multiplying the descriptions by $P_\text{UPSIFT}^k$.
    That is, we can steer the Upright SIFT descriptions with $P_\text{UPSIFT}$.
\end{exmp}

It is known from representation theory~\cite{serreLinearRepresentationsFinite1977} what all possible representations of $C_4$ are,
and hence what all possible steerers for rotation equivariant descriptors are.
As this result will be necessary for the remainder of the text,
we collect it in a theorem.
Similar results are also known for other groups \eg the continuous rotation group $\mathrm{SO}(2)$,
which we discuss in the next section.

\begin{theorem}[Representations of $C_4$]
\label{thm:c4-rep}
    Let $\rho$ be a representation of $C_4$ on $\mathbb{R}^D$. Then, there exists an invertible matrix $Q$ and $j_d\in\{0, 1, 2, 3\}$ such that
    \begin{equation}\label{eq:eigdecom_generator}
        \rho(\mathbf g^k) = Q^{-1}
        \mathrm{diag}(\mathbf{i}^{kj_1}, \mathbf{i}^{kj_2}, \ldots, \mathbf{i}^{kj_D})
        Q.
    \end{equation}
\end{theorem}
The diagonal in~\eqref{eq:eigdecom_generator} contains the eigenvalues of $\rho(\mathbf{g}^k)$.
\begin{exmp}
    The Upright SIFT steerer $P_\text{UPSIFT}$ is diagonalizable with an equal amount of each eigenvalue $\pm 1$, $\pm \mathbf{i}$.
\end{exmp}

The complex eigenvalues must appear in conjugate pairs as we take $\rho(\mathbf g)$ to be real-valued.
It is then possible to do a change of basis so that each pair $\mathbf{i}, -\mathbf{i}$ on the diagonal in \eqref{eq:eigdecom_generator}
is replaced by a block $\smallimagblock$.
In this way, $\rho(\mathbf g)$ can always be block-diagonalized: $\rho(\mathbf g) = Q^{-1} B Q$ where
$Q$ and $B$ are real valued and $B$ is block-diagonal with blocks of sizes~1 and~2.

\subsection{Representation theory of \texorpdfstring{$\mathrm{SO}(2)$}{SO(2)}}
$\mathrm{SO}(2)$ is a one-parameter Lie group, \ie a continuous group with one degree of freedom $\alpha$---the rotation angle.
A $D$-dimensional representation of $\mathrm{SO}(2)$ is a map $\varsigma: [0, 2\pi)\to \mathrm{GL}(\mathbb{R}, D)$ such that addition modulo $2\pi$ on the input is encoded as matrix multiplication on the output---we will consistently use $\varsigma$ for $\mathrm{SO}(2)$ representations to separate them from $C_4$ representations $\rho$ ($\rho$ will also be used for representations of general groups).
The most familiar is the two-dimensional representation $\varsigma(\alpha)=\begin{psmallmatrix}
    \cos(\alpha) & -\sin(\alpha) \\ \sin(\alpha) & \cos(\alpha)
\end{psmallmatrix}$
which rotates 2D coordinates.
Similar to the $C_4$ case in Theorem~\ref{thm:c4-rep}, we can write down a general representation for $\mathrm{SO}(2)$ as follows~\cite{Woit2017}.
\begin{theorem}[Representations of $\mathrm{SO}(2)$] Let $\varsigma$ be a representation of $\mathrm{SO}(2)$ on $\mathbb{R}^D$.
Then there exists an invertible $Q$ and $j_d\in\mathbb{Z}$ such that
\begin{equation}\label{eq:so2-decom}
    \varsigma(\alpha) = Q^{-1} \mathrm{diag}\left(\mathrm{e}^{\mathbf{i} j_1 \alpha}, \mathrm{e}^{\mathbf{i} j_2 \alpha}, \ldots, \mathrm{e}^{\mathbf{i} j_D \alpha}\right) Q.
\end{equation}
\end{theorem}
The $j_d$'s are the frequencies of the eigenspaces of $\varsigma$.
Complex eigenvalues appear in conjugate pairs
so \eqref{eq:so2-decom} can be rewritten as a block diagonal decomposition
$\varsigma(\alpha) = Q^{-1} B Q$ where $Q$ and $B$ are real valued and $B$ is block-diagonal with minimal blocks.
The admissible blocks ($=$ real-valued irreducible representations) in $B$ are then the $1\times 1$ block $\begin{pmatrix} 1 \end{pmatrix}$ and the $2\times 2$ blocks
\begin{equation}\label{eq:so2_blocks}
\begin{pmatrix}
    \cos(j\alpha) & -\sin(j\alpha) \\ \sin(j\alpha) & \cos(j\alpha)
\end{pmatrix} \quad \text{for $j\in\mathbb{Z}\setminus\{0\}$.}
\end{equation}
We can write 
$\varsigma(\alpha) = \expm\left(\alpha Q^{-1} \mathrm{diag}(\mathbf{i}j_1,\ldots,\mathbf{i}j_D) Q\right)$ where $\expm$ is the matrix exponential.
The quantity $\mathrm{d}\varsigma := Q^{-1} \mathrm{diag}(\mathbf{i}j_1,\ldots,\mathbf{i}j_D) Q$ is called a Lie algebra representation of $\mathrm{SO}(2)$, here in its most general form.
When training a steerer for $\mathrm{SO}(2)$ it is practical to train a $D\times D$ matrix $\mathrm{d}\varsigma$ and steer using $\varsigma(\alpha)=\expm(\alpha \mathrm{d}\varsigma)$.

\subsection{Disentangling description space}
When we have a steerer, we get a description space on which rotations act---up to a change of basis---by a block-diagonal matrix $\oplus_{j=1}^J B_j$. 
The description space can then be thought of as being disentangled into different subspaces
where rotations act in different ways $B_j$~\cite{cohen14irreps, Worrall_2017_ICCV}.
We detail what this means
for keypoint descriptors in Appendix~\ref{app:theory}.

\section{Descriptors and steerers}
\label{sec:desc_and_steer}
This work's crucial observation and assumption is that learned descriptors, while not invariant, are approximately equivariant so that they have a steerer.
Or, as a weaker assumption, they can be trained to be equivariant. It may seem that this is a strong assumption. However, a seemingly less strong assumption turns out to be equivalent.

\begin{restatable}{theorem}{thmgupta}[Adapted from \citet{shakerinava2022structuring}, \citet{gupta2023structuring}]
\label{thm:gupta}
    Assume that we have a function $f: V \to \mathbb{S}^{D-1}$ and a group $G$ with 
    representation $\rho_\text{in}$ on $V$ such that, for all $v,v'\in V$ and $g\in G$
    \begin{equation}\label{eq:gupta}
        \langle f(\rho_{\text{in}}(g)v), f(\rho_{\text{in}}(g)v')\rangle = \langle f(v), f(v')\rangle.
    \end{equation}
    Then there exists an orthogonal representation $\rho(g)$, %
    such that $f$ is equivariant w.r.t. $G$ with representations $\rho_\text{in}$ and $\rho$.
\end{restatable}
We provide a proof in Appendix~\ref{app:theory}.
Since we match normalized keypoint descriptions by their cosine similarity, Theorem~\ref{thm:gupta} is highly applicable to the image matching problem.
If a keypoint descriptor $f$ is perfectly consistent in the matching scores when \emph{simultaneously} rotating the images, then the scalar products in \eqref{eq:gupta} will be equal
so that the theorem tells us that $f$ has a steerer $\rho$.
Furthermore, we can expect many local image features
to appear in all orientations even over a dataset of upright images, thus encouraging \eqref{eq:gupta} to hold for $f$ trained on large datasets.

\subsection{Three settings for investigating steerers} \label{sec:settings}
As $C_4$ is cyclic, all its representations are defined by $\rho(\mathbf{g})$,
where $\mathbf{g}$ is the generator of $C_4$.
To find a steerer for a keypoint descriptor under $C_4$ hence comes down to finding a single matrix $\rho(\mathbf{g})$ that represents rotations by $\ang{90}$ in the description space.
Similarly, for $\mathrm{SO}(2)$ we find the single
matrix $\mathrm{d}\varsigma$ that defines the representation $\varsigma$.

We will consider three settings. In each case we will optimize $\rho(\mathbf{g})$ and/or $f$ over
the MegaDepth training set \cite{li2018megadepth} with rotation augmentation and maximize
\begin{equation} \label{eq:steer_objective}
p\left(f(P_{90}^{k_1}I_1, R_{90}^{k_1}x_1), \rho(\mathbf{g})^k f(P_{90}^{k_2}I_2, R_{90}^{k_2}x_2)\right)
\end{equation}
where $p$ is the matching probability \eqref{eq:match_prob}.
The number of rotations $k_1$ and $k_2$ for each image are sampled independently during training, and $k=k_1 - k_2 \mod{4}$ is the number of rotations that aligns image $I_2$ to image $I_1$.
Thus, $\rho(\mathbf{g})^k$ aligns the relative rotation between descriptions in $I_2$ and $I_1$.
We optimize continuous steerers $\varsigma$ %
analogously to \eqref{eq:steer_objective}.

\begin{enumerate}[font=\bfseries, wide, labelwidth=!, labelindent=0pt]
\item[Setting A: Fixed descriptor, optimized steerer.] If a descriptor works equally well for upright images as well as images rotated the same amount from upright, then according to Theorem~\ref{thm:gupta}, we should expect that there exists a steerer $\rho(\mathbf{g})$ such that \eqref{eq:equiv_desc} holds.
To find $\rho(\mathbf{g})$ we optimize it as a single $D\times D$ linear layer by maximizing~\eqref{eq:steer_objective}.

\item[Setting B: Joint optimization of descriptor and steerer. ]
The aim is to find a steerer that is as good as possible for the given data.
We will see in the experiments, by looking at the evolution of the eigenvalues of $\rho(\mathbf{g})$ during training, that this joint optimization has many local optima and is highly dependent on the initialization of $\rho(\mathbf{g})$.
However, looking at the eigenvalues of $\rho(\mathbf{g})$ does give knowledge about which descriptor dimensions are most important, as will be explained in Section~\ref{sec:training_dynamics}.

\item[Setting C: Fixed steerer, optimized descriptor.]
To get the most precise control over the rotation behaviour of a descriptor, we can
fix the steerer and optimize only the descriptor.
This enables us to investigate how much influence the choice of steerer has on the descriptor.
For instance, choosing the steerer as the identity leads to a rotation invariant descriptor.
We will see in the experiments that this choice leads to suboptimal
performance on upright images compared to other steerers.
\end{enumerate}

\begin{table}
 \small
     \centering
     \caption{\textbf{Evaluation on Roto-360~\cite{lee2023learning}. }
     We report the percentage of correct matches at three thresholds. %
     We use the DeDoDe-SO2 detector with $5,000$ keypoints in the last two rows. 
     Matching strategies are Max Matches for the $C_4$-descriptor and Max Sim. over $C_8$ for the $\mathrm{SO}(2)$-descriptor.
     See Section~\ref{sec:models_considered} for the shorthands for our models.
     }
     \begin{tabular}{
        ll
        rrr
     }

     \toprule
      Detector & Descriptor & 3px & 5px & 10px 
      \\
    
    \midrule
    SIFT \cite{lowe2004distinctive} & SIFT \cite{lowe2004distinctive} & 78 & 78 & 79 \\
    ORB \cite{rublee2011orb} & ORB \cite{rublee2011orb} & 79 & 85 & 87 \\
         SuperPoint \cite{detone2018superpoint} & RELF, single \cite{lee2023learning} &    
            90 & 91 & 93 
    \\
         SuperPoint \cite{detone2018superpoint} & RELF, multiple \cite{lee2023learning} &    
            92 & 93 & 94
    \\
        SuperPoint \cite{detone2018superpoint} & C4-B (ours) &  
            82 & 82 & 83
    \\
       SuperPoint \cite{detone2018superpoint} & SO2-Spread-B (ours) &
       \textbf{96} & \textbf{97} & 97
    \\
    \midrule
    DeDoDe \cite{edstedt2024dedode} & C4-B (ours) &  
            82 & 84 & 86
    \\
DeDoDe \cite{edstedt2024dedode}        & SO2-Spread-B (ours) &
       95 & \textbf{97} & \textbf{98}
    \\
     \bottomrule
     \end{tabular}
     \label{tab:roto360}
\end{table}

\subsection{Matching with equivariant descriptions}
\label{sec:matching_strat}
This section presents several approaches to rotation invariant matching using equivariant descriptors.
Throughout, we will denote the $D$-dimensional descriptions of $N$ keypoints in two images $I_1,I_2$ by
$y_1,y_2\in\mathbb{R}^{D\times N}$ and will assume
that we know the $C_4$-steerer $\rho(\mathbf{g})$ that rotates descriptions $\ang{90}$ or the $\mathrm{SO}(2)$-steerer $\varsigma(\alpha)$ through the
Lie algebra generator $\mathrm{d}\varsigma$.
For matching we follow DeDoDe~\cite{edstedt2024dedode},
as described in Section~\ref{sec:intro:matching}.
The base similarity used is the cosine similarity, so we compute
$y_1^T y_2$ for normalized descriptions to get an $N\times N$ matrix of pairwise scores on
which dual softmax \eqref{eq:match_prob} is applied. 
Matches are mutual most similar descriptions with similarity above $0.01$.

\begin{enumerate}[font=\bfseries, wide, labelwidth=!,labelindent=0pt]
\item[Max matches over steered descriptions.]
The first way of obtaining 
invariant matches is to match $y_1$ with $\rho(\mathbf{g})^k y_2$ for $k=0,1,2,3$
and keep the matches from the $k$ that has the most matches.
This is similar to matching the image $I_1$ with four different
rotations of $I_2$ but alleviates the need for rerunning the descriptor network
for each rotation. 

\item[Max similarity over steered descriptions.]
A computationally cheaper version is to select the matching matrix not as $y_1^Ty_2$ but as $\max_{k} y_1^T\rho(\mathbf{g})^k y_2$,
where the $\max$ is elementwise over the matrix.

\item[$\mathbf{SO(2)}$-steerers. ] To apply the above matching strategies to $\mathrm{SO}(2)$-steerers $\varsigma(\alpha)=\expm(\alpha\mathrm{d}\varsigma)$ we discretize $\varsigma$.
A $C_\ell$-steerer is obtained through $\rho(\mathbf{g}_\ell)=\expm\left(\frac{2\pi}{\ell}\mathrm{d}\varsigma\right)$, where $\mathbf{g}_\ell$ generates $C_\ell$.
We will use $C_8$ in the experiments.

\label{sec:speed1}
\item[Procrustes matcher. ]
If all eigenvalues of $\rho(\mathbf{g})$ are $\pm\mathbf{i}$, the steerer can be block-diagonalized with only the block $\smallimagblock$\footnote{This also holds for $\mathrm{SO}(2)$ steerers, referring to eigenvalues and blocks of the Lie algebra generator $\mathrm{d}\varsigma$.}.
The descriptions consist of $D/2$ two-dimensional quantities that all
rotate with the same frequency as the image.
We will refer to them as frequency 1 descriptions
and view them reshaped as 
$y\in\mathbb{R}^{2\times (D/2) \times N}$.
A 2D rotation matrix acts on these descriptions from the left when the image rotates, and
we can find the optimal rotation matrix $R_{m,n}$ that aligns each pair $y_{1,m},y_{2,n}\in\mathbb{R}^{2\times (D/2)}$ by solving the Procrustes problem via SVD.
The matching matrix is obtained by computing $\langle R_{m, n} y_{1,m},y_{2,n}\rangle$ for each pair.
$R_{m, n}$ gives the relative rotation between each pair of keypoints,
which can be useful 
\eg for minimal relative pose solvers
\cite{barath2020making, barath2022relative} or for outlier filtering \cite{cavalli2020adalam}.
We leave exploring this per-correspondence geometry to future work.

\end{enumerate}

\begin{table}
 \small
     \centering
     \caption{\textbf{Evaluation on AIMS \cite{stoken2023astronaut}}. 
         We report the average precision (AP) in percent on different splits of AIMS: ``North Up'' (N. Up) contains images with small rotations, ``All Others'' (A. O.) contains images with larger rotations and ``All'' contains all images.
         We use the DeDoDe-SO2 detector and $10,000$ keypoints throughout.
          See Section~\ref{sec:models_considered} for the shorthands for our models.
     }
     \begin{tabular}{
        l
        rrr
     }
     \toprule
       Method & 
      N. Up & A. O. & All
      \\
    \midrule
    SE2-LoFTR \cite{bokman2022case} & 58 & 51 & 52 \\
    C4-B, Max Matches (ours) & 52 & 51 & 51 \\
    SO2-Spread-B, Max Sim. C8 (ours) & 60 & 57 & 58 \\
    SO2-Freq1-B, Procrustes (ours) & \textbf{64} & \textbf{59} & \textbf{60}\\
     \bottomrule
     \end{tabular}
     \label{tab:aims}
\end{table}

\begin{table*}
 \small
     \centering
     \caption{\textbf{Evaluation on MegaDepth~\cite{li2018megadepth,sun2021loftr}}. 
     The first section shows Setting~A where we only optimize the steerer, the second section shows Setting~B where we jointly optimize the descriptor and steerer and the third section shows the Setting~C where we predefine the steerer and optimize only the descriptor.
     For MegaDepth-1500 we always use dual softmax matcher to evaluate the descriptors on upright images.
     We use $20,000$ keypoints throughout. The best values for \textcolor{Blue}{\textbf{B}}- and \textcolor{OliveGreen}{\textbf{G}}-models are highlighed in each column.
      See Section~\ref{sec:models_considered} for shorthand explanations for our models.
     A larger version of this table with more methods is available in Appendix~\ref{app:more_experiments}.
     }
     \begin{tabular}{
        lll
        rrr 
        rrr
        rrr
     }
     \toprule
      Detector & Descriptor & & \multicolumn{3}{l}{\phantom{AUC @ }MegaDepth-1500} &  \multicolumn{3}{l}{MegaDepth-C4} &  \multicolumn{3}{l}{MegaDepth-SO2} \\ 
      DeDoDe & DeDoDe & Matching strategy & AUC $@$
      ~$5^{\circ}$&$10^{\circ}$&$20^{\circ}$ & 
      ~$5^{\circ}$&$10^{\circ}$&$20^{\circ}$ & 
      ~$5^{\circ}$&$10^{\circ}$&$20^{\circ}$\\
    \midrule
         Original & B & Dual softmax & 
            49 & 65 & 77 &
            12 & 17 & 20 &
            12 & 16 & 20
        \\
         Original & B & Max matches C4-steered & 
            \dittotikz & \dittotikz & \dittotikz & 
            43 & 60 & 73 &
            30 & 44 & 56 
        \\
        SO2 & B & Max matches C8-steered & 
            50 & 66 & 78 &
            40 & 57 & 70 &
            34 & 51 & 65
        \\
         Original & G & Dual softmax & 
            \textbf{\textcolor{OliveGreen}{52}} & \textbf{\textcolor{OliveGreen}{69}} & \textbf{\textcolor{OliveGreen}{81}} &
            13 & 17 & 21 &
            16 & 22 & 28
        \\
         Original & G & Max matches C4-steered & 
            \dittotikz & \dittotikz & \dittotikz & 
            31 & 45 & 57 &
            26 & 39 & 50 
        \\
    \midrule
        C4 & C4-B & Max matches C4-steered &
            \textbf{\textcolor{Blue}{51}} & \textbf{\textcolor{Blue}{67}} & \textbf{\textcolor{Blue}{79}} &
            \textbf{\textcolor{Blue}{50}} & \textbf{\textcolor{Blue}{67}} & \textbf{\textcolor{Blue}{79}} &
            39 & 55 & 68 
        \\
        SO2 & SO2-B & Max matches C8-steered &
            47 & 63 & 76 &
            47 & 63 & 76 &
            44 & 61 & 74 
        \\
        SO2 & SO2-Spread-B & Max matches C8-steered &
            50 & 66 & \textbf{\textcolor{Blue}{79}} &
            49 & 66 & 78 &
            \textbf{\textcolor{Blue}{46}} & \textbf{\textcolor{Blue}{63}} & \textbf{\textcolor{Blue}{76}} 
        \\
        SO2 & SO2-Spread-B & Max similarity C8-steered &
            49 & 66 & 78 &
            47 & 64 & 77 &
            43 & 61 & 74 
        \\
    \midrule
        C4 & C4-Inv-B & Dual softmax &
            48 & 64 & 76 &
            47 & 63 & 76 &
            39 & 55 & 69 
        \\
        C4 & C4-Perm-B & Max matches C4-steered &
            50 & \textbf{\textcolor{Blue}{67}} & \textbf{\textcolor{Blue}{79}} &
            \textbf{\textcolor{Blue}{50}} & 66 & \textbf{\textcolor{Blue}{79}} &
            39 & 54 & 67 
        \\
        SO2 & SO2-Inv-B & Dual softmax &
            46 & 62 & 75 &
            45 & 61 & 74 &
            43 & 60 & 73 
        \\
        SO2 & SO2-Freq1-B & Max matches C8-steered &
            47 & 64 & 77 &
            47 & 64 & 76 &
            45 & 62 & 75 
        \\
        SO2 & SO2-Freq1-B & Procrustes &
            47 & 64 & 76 &
            46 & 62 & 75 &
            45 & 61 & 74 
        \\
        C4 & C4-Perm-G & Max matches C4-steered &
            \textbf{\textcolor{OliveGreen}{52}} & \textbf{\textcolor{OliveGreen}{69}} & \textbf{\textcolor{OliveGreen}{81}} &
            \textbf{\textcolor{OliveGreen}{53}} & \textbf{\textcolor{OliveGreen}{69}} & \textbf{\textcolor{OliveGreen}{82}} &
            \textbf{\textcolor{OliveGreen}{44}} & \textbf{\textcolor{OliveGreen}{61}} & \textbf{\textcolor{OliveGreen}{74}} 
        \\
     \bottomrule
     \end{tabular}
     \label{tab:megadepth}
\end{table*}

\section{Experiments}
\label{sec:experiments}

We train and evaluate a variety of descriptors and steerers.
Experimental details are covered in Appendix~\ref{app:experiment_details}.
We provide comparisons to TTA in performance and runtime in Appendix~\ref{app:more_experiments}, as well as experiments with more matching strategies and an explicit experiment to test the connection between Theorem~\ref{thm:gupta} and rotation equivariance.

We start by reporting results on two public benchmarks for rotation invariant image matching.
Then, we will present ablation results for the MegaDepth benchmark, both for the standard version with upright images and a version where we have rotated the input images.

\subsection{Models considered }
\label{sec:models_considered}
Our base models are the DeDoDe-B
and DeDoDe-G descriptors introduced in \cite{edstedt2024dedode}.
These are both $D=256$ dimensional descriptors.
The focus will be on the smaller model DeDoDe-B, as this gives us the chance
to do large-scale ablations.
We train all models on MegaDepth~\cite{li2018megadepth}.
To obtain rotation-consistent detections, we retrain two versions of the DeDoDe-detector, with data augmentation over $C_4$ and $\mathrm{SO}(2)$ respectively, denoted DeDoDe-\{C4, SO2\}.

For Setting~B (Section~\ref{sec:settings}), we will see that the initialization of the steerer matters.
Similarly, for Setting~C, we can fix the steerer with different eigenvalue structures.
Here, we introduce shorthand, which is used in the result tables.
We will refer to the case of all eigenvalues 1 as \textit{Inv} for invariant.
This case corresponds to the ordinary notion of data augmentation, where
the descriptions for rotated images should be the same as for non-rotated images.
The case when all eigenvalues are $\pm\mathbf{i}$ is denoted \textit{Freq1} for frequency 1 as explained
in Section~\ref{sec:speed1}.
For $C_4$-steerers, the case with an equal distribution of all eigenvalues
$\pm 1, \pm\mathbf{i}$ will be denoted \textit{Perm},
as this is the eigenvalue signature of a cyclic permutation of order 4.
For $\mathrm{SO}(2)$-steerers, the case with an equal distribution of eigenvalues
$0, \pm\mathbf{i}, \pm2\mathbf{i}, \ldots \pm6\mathbf{i}$ will be denoted \textit{Spread} (the cutoff $6$ was arbitrarily chosen).
The \textit{Perm} and \textit{Spread} steerers correspond to a broad range of frequencies in the description space.
When none of the above labels (\emph{Inv}, \emph{Freq1}, \emph{Perm} or \emph{Spread}) is attached to a descriptor and steerer trained jointly 
in Setting~B, then we initialize the steerer with values uniformly in $(-D^{-1/2}, D^{-1/2})$\footnote{This is the standard initialization of a linear layer in Pytorch~\cite{paszke2019pytorch, he2015delving}.}, the eigenvalues are then approximately uniformly distributed in the disk with radius $3^{-1/2}$ \cite{tao2010random}.
When a descriptor is trained with $k\cdot\ang{90}$ rotations, we append \textit{C4} to its name and when trained
with continuous augmentations, we append \textit{SO2} to its name.

\subsection{Roto-360}\label{sec:roto-360}
We evaluate on the Roto-360 benchmark~\cite{lee2023learning}, which consists of ten image pairs from HPatches \cite{balntas2017hpatches} where the second image in each image pair is rotated by all multiplies of $\ang{10}$ to obtain 360 image pairs in total. We report the average precision of the obtained matches and compare it to the current state-of-the-art RELF~\cite{lee2023learning}.
The results are shown in Table~\ref{tab:roto360}.
We see that we outperform RELF when using
methods trained for continuous rotations.
Our matching runs around three times faster than RELF on Roto-360.

\subsection{AIMS}\label{sec:aims}
The Astronaut Image Matching Subset (AIMS) \cite{stoken2023astronaut} consists of
images taken by astronauts from the ISS and satellite images covering the broad regions that the astronaut images could depict.
The task consists of finding the pairs of astronaut images and satellite images that show the same
locations on Earth.
Pairs are found by setting a threshold for the number of matches between images after homography estimation with RANSAC.

The relative rotations of the astronaut and satellite images are unknown, making the task suitable for rotation-invariant matchers.
Indeed, in \cite{stoken2023astronaut}, the best performing method is the rotation invariant SE2-LoFTR~\cite{bokman2022case}, which we compare to.
The AIMS can be split into ``North Up'' astronaut images, consisting of images with small rotations (between $\ang{0}$ and $\ang{90}$) and ``All Others'', consisting of images with large rotations.
This split further enables the evaluation of rotation invariant matchers.
We report the average precision
over the whole dataset, as opposed to the approach in \cite{stoken2023astronaut}, where the score is computed over at most 100 true negatives per astronaut image.
Results are shown in Table~\ref{tab:aims}. Further,
we plot precision-recall curves in Appendix~\ref{app:more_experiments}.
We generally outperform SE2-LoFTR, particularly on the heavily rotated images in ``All Others''.

\subsection{MegaDepth-1500}\label{sec:megadepth}
We evaluate on a held-out part of MegaDepth (MegaDepth-1500 following \cite{sun2021loftr}).
Here, the task is to take two input images and output the relative pose between the cameras.
The performance is measured by the AUC of the pose error.
Additionally, we create two versions of MegaDepth with rotated images to
evaluate the rotational robustness of our models.
For MegaDepth-C4, the second image in every image pair is rotated $(i \mod 4)\cdot \ang{90}$ where $i$ is the index of the image pair.
We visualize a pair in MegaDepth-C4 in Appendix~\ref{app:more_experiments}, illustrating the improvement from DeDoDe-B to DeDoDe-B with a steerer optimized in Setting~A.  %
For MegaDepth-SO2, the second image in every image pair is instead rotated $(i \mod 36)\cdot \ang{10}$, thus requiring robustness under continuous rotations.

The results are presented in Table~\ref{tab:megadepth}; for more methods, see Appendix~\ref{app:more_experiments}.
We summarize the main takeaways:
\begin{enumerate}
    \item It is possible to find steerers for the original DeDoDe models (\eg the second row of the table), even though they were not trained with any rotation augmentation.
    \item The trained $C_4$ steerers perform very well as their scores on MegaDepth-1500 and MegaDepth-C4 are the same.
    \item Training DeDoDe-B jointly with a $C_4$ steerer (C4-B) or with a fixed steerer (C4-Perm-B) improves results on upright images---this can be attributed to the fact that training with a steerer enables using rotation augmentation.
    \item The right equivariance for the task at hand is crucial---$\mathrm{SO}(2)$-steerers outperform others on MegaDepth-SO2.
    \item The eigenvalue distribution of the steerer is important---invariant models are worse than others, and SO2-B and SO2-Freq1-B are worse than SO2-Spread-B.
    \item DeDoDe-G can be made equivariant (C4-Perm-G), even though it has a frozen DINOv2~\cite{oquab2023dinov2} ViT backbone.
\end{enumerate}

\begin{figure}
    \centering
    \includegraphics[width=\columnwidth]{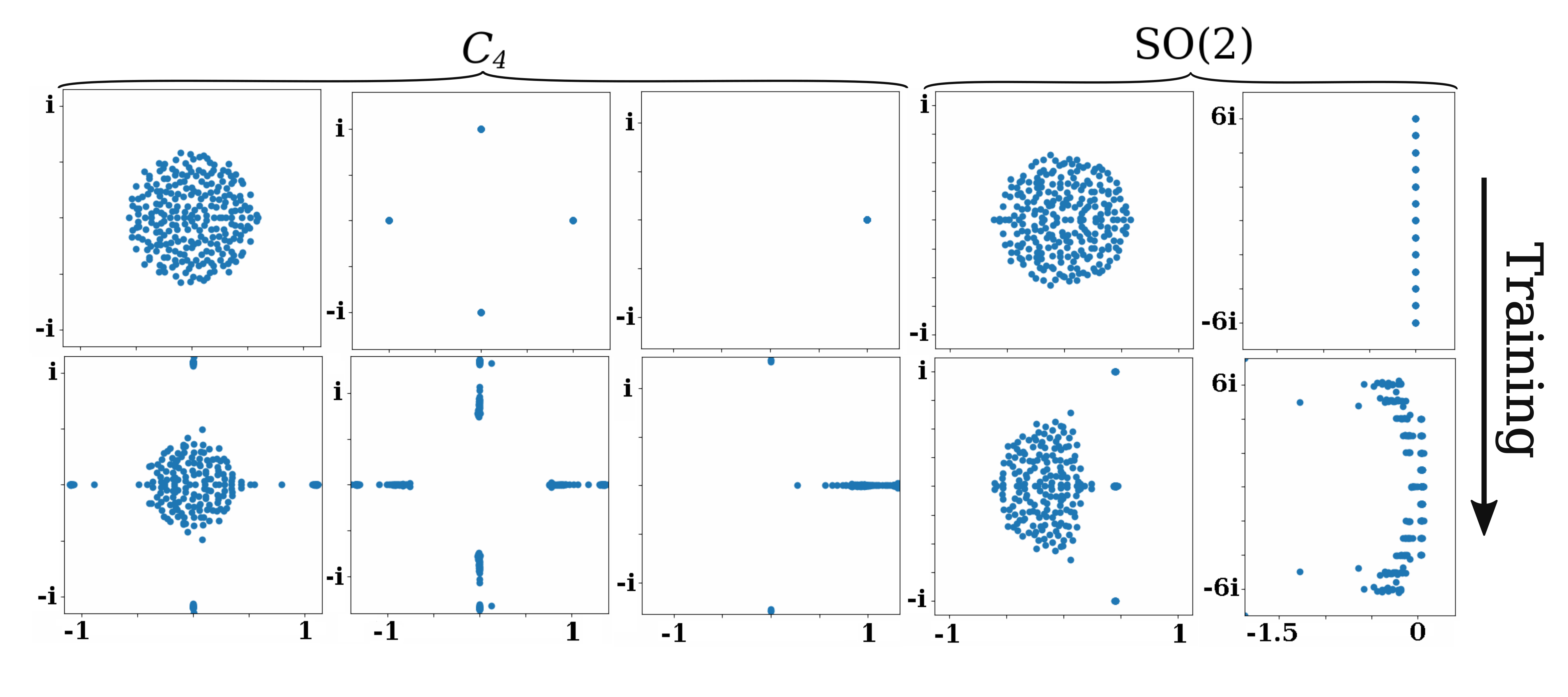}

     \caption{\textbf{Training evolution of eigenvalue distributions of steerers.}
    We plot the eigenvalue distribution of $C_4$-steerers $\rho(\mathbf{g})$ (first three columns) and Lie algebra generators $\mathrm{d}\varsigma$ for $\mathrm{SO}(2)$-steerers (last two columns) in the complex plane, with different initializations when trained jointly with a descriptor.
    The top row depicts the eigenvalues at the start, and the bottom
    row at the end of training.
    There are $D=256$ eigenvalues in every plot---many congregate at the ``admissible'' eigenvalues as described in Section~\ref{sec:equiv_steer}---but some do not, see the discussion in Section~\ref{sec:training_dynamics}.
    These visualizations highlight the initialization sensitivity of the steerer.
    We show gif movies of the training evolution at \href{https://github.com/georg-bn/rotation-steerers}{this https url}.
    }
    \label{fig:eigval}
\end{figure}

\subsection{Training dynamics of steerer eigenvalues}
\label{sec:training_dynamics}
This section aims to demonstrate that joint optimization of the steerer and descriptor does not necessarily lead to a good eigenvalue structure for the steerer.
We plot the evolution of the eigenvalues of the steerer over the training epochs in Figure~\ref{fig:eigval}.
For $C_4$-steerers we plot the eigenvalues of $\rho(\mathbf{g})$ itself,
while for $\mathrm{SO}(2)$-steerers we plot the eigenvalues $\lambda_d$ of the Lie algebra generator $\mathrm{d}\varsigma$, so that the eigenvalues of the steerer $\varsigma(\alpha)$ are $\mathrm{e}^{\alpha \lambda_d}$. 
It is clear from Figure~\ref{fig:eigval} that the initialization of the steerer
influences the final distribution of eigenvalues a lot
and we saw in Table~\ref{tab:megadepth} that the eigenvalue distribution
of the steerer matters for performance.
Thus, we think it is an important direction for future work to figure out how to get around this initialization sensitivity.
The choice of eigenvalue structure is related to the problem of specifying which group representations to use in
the layers of equivariant neural networks in general.

As a side effect of plotting the eigenvalues, we find that some of the steerer's eigenvalues have much lower absolute values than others\footnote{
The absolute values of the eigenvalues of a $\mathrm{SO}(2)$ steerer $\varsigma(\alpha)$ are $\mathrm{e}^{\alpha\mathrm{Re}(\lambda_d)}$ where $\lambda_d$ are the eigenvalues of $\mathrm{d}\varsigma$
that are plotted in the two rightmost columns of Figure~\ref{fig:eigval}.
Therefore, a lower real value of $\lambda_d$ means a lower absolute value of the eigenvalue of the steerer.
}.
The steerer is applied to descriptions before they are normalized,
so the absolute value of the maximum eigenvalue is unimportant, but
the relative size of the eigenvalues tells us something about feature importance.
Eigenvectors with small eigenvalues cannot be too important for matching,
since they will be relatively downscaled when applying the steerer in the
optimization of \eqref{eq:steer_objective}.
Indeed, small eigenvalues seem to correspond to unimportant dimensions of the descriptor---we maintain matching performance when projecting the descriptions to the span of the eigenvectors with large eigenvalues.
This is related to PCA for dimensionality reduction, which
has successfully been used for classical keypoint descriptors~\cite{ke2004pca}.

\section{Conclusion}
We developed a new framework for rotation equivariant keypoint descriptors using steerers---linear maps that encode image rotations in description space.
After outlining the general theory of steerers using %
representation theory, we designed a large set of experiments with steerers in three settings: (A) optimizing a steerer for a fixed descriptor, (B) optimizing a steerer and a descriptor jointly and (C) optimizing a descriptor for a fixed steerer.
Our best models obtained new state-of-the-art results on the rotation invariant matching benchmarks Roto-360 and AIMS.

\section*{Acknowledgements}
We thank Alex Stoken for help with the evaluation on AIMS.
This work was supported by the Wallenberg Artificial
Intelligence, Autonomous Systems and Software Program
(WASP), funded by the Knut and Alice Wallenberg Foundation and by the strategic research environment ELLIIT, funded by the Swedish government. The computational resources were provided by the
National Academic Infrastructure for Supercomputing in
Sweden (NAISS) at C3SE, partially funded by the Swedish Research
Council through grant agreement no.~2022-06725, and by
the Berzelius resource, provided by the Knut and Alice Wallenberg Foundation at the National Supercomputer Centre.

\newpage
{
    \small
    \bibliographystyle{ieeenat_fullname}
    \bibliography{main}

\begin{thebibliography}{61}
\providecommand{\natexlab}[1]{#1}
\providecommand{\url}[1]{\texttt{#1}}
\expandafter\ifx\csname urlstyle\endcsname\relax
  \providecommand{\doi}[1]{doi: #1}\else
  \providecommand{\doi}{doi: \begingroup \urlstyle{rm}\Url}\fi

\bibitem[Baatz et~al.(2010)Baatz, K{\"o}ser, Chen, Grzeszczuk, and
  Pollefeys]{baatz2010handling}
Georges Baatz, Kevin K{\"o}ser, David Chen, Radek Grzeszczuk, and Marc
  Pollefeys.
\newblock Handling urban location recognition as a 2d homothetic problem.
\newblock In \emph{Computer Vision--ECCV 2010: 11th European Conference on
  Computer Vision, Heraklion, Crete, Greece, September 5-11, 2010, Proceedings,
  Part VI 11}, pages 266--279. Springer, 2010.

\bibitem[Bagad et~al.(2022)Bagad, Eijkelboom, Fokkema, de~Goede, Hilders, and
  Kofinas]{bagad2022c}
Piyush Bagad, Floor Eijkelboom, Mark Fokkema, Danilo de Goede, Paul Hilders,
  and Miltiadis Kofinas.
\newblock C-3po: Towards rotation equivariant feature detection and
  description.
\newblock In \emph{European Conference on Computer Vision}, pages 694--705.
  Springer, 2022.

\bibitem[Balntas et~al.(2017)Balntas, Lenc, Vedaldi, and
  Mikolajczyk]{balntas2017hpatches}
Vassileios Balntas, Karel Lenc, Andrea Vedaldi, and Krystian Mikolajczyk.
\newblock {HPatches}: A benchmark and evaluation of handcrafted and learned
  local descriptors.
\newblock In \emph{Proceedings of the IEEE conference on computer vision and
  pattern recognition}, pages 5173--5182, 2017.

\bibitem[Barath and Kukelova(2022)]{barath2022relative}
Daniel Barath and Zuzana Kukelova.
\newblock Relative pose from sift features.
\newblock In \emph{European Conference on Computer Vision}, pages 454--469.
  Springer, 2022.

\bibitem[Barath et~al.(2020{\natexlab{a}})Barath, Noskova, Ivashechkin, and
  Matas]{barath2020magsac++}
Daniel Barath, Jana Noskova, Maksym Ivashechkin, and Jiri Matas.
\newblock Magsac++, a fast, reliable and accurate robust estimator.
\newblock In \emph{Proceedings of the IEEE/CVF conference on computer vision
  and pattern recognition}, pages 1304--1312, 2020{\natexlab{a}}.

\bibitem[Barath et~al.(2020{\natexlab{b}})Barath, Polic, F{\"o}rstner, Sattler,
  Pajdla, and Kukelova]{barath2020making}
Daniel Barath, Michal Polic, Wolfgang F{\"o}rstner, Torsten Sattler, Tomas
  Pajdla, and Zuzana Kukelova.
\newblock Making affine correspondences work in camera geometry computation.
\newblock In \emph{European Conference on Computer Vision}, pages 723--740.
  Springer, 2020{\natexlab{b}}.

\bibitem[Bay et~al.(2006)Bay, Tuytelaars, and Van~Gool]{bay2006surf}
Herbert Bay, Tinne Tuytelaars, and Luc Van~Gool.
\newblock Surf: Speeded up robust features.
\newblock In \emph{European conference on computer vision}, pages 404--417.
  Springer, 2006.

\bibitem[Bay et~al.(2008)Bay, Ess, Tuytelaars, and {Van Gool}]{bay2008surf}
Herbert Bay, Andreas Ess, Tinne Tuytelaars, and Luc {Van Gool}.
\newblock Speeded-up robust features (surf).
\newblock \emph{Computer Vision and Image Understanding}, 110\penalty0
  (3):\penalty0 346--359, 2008.
\newblock Similarity Matching in Computer Vision and Multimedia.

\bibitem[B{\"o}kman and Kahl(2022)]{bokman2022case}
Georg B{\"o}kman and Fredrik Kahl.
\newblock A case for using rotation invariant features in state of the art
  feature matchers.
\newblock In \emph{Proceedings of the IEEE/CVF Conference on Computer Vision
  and Pattern Recognition}, pages 5110--5119, 2022.

\bibitem[B{\"o}kman and Kahl(2023)]{bokman2023investigating}
Georg B{\"o}kman and Fredrik Kahl.
\newblock Investigating how {Re{LU}}-networks encode symmetries.
\newblock In \emph{Thirty-seventh Conference on Neural Information Processing
  Systems}, 2023.

\bibitem[Bronstein et~al.(2021)Bronstein, Bruna, Cohen, and
  Veli{\v{c}}kovi{\'c}]{bronstein2021geometric}
Michael~M Bronstein, Joan Bruna, Taco Cohen, and Petar Veli{\v{c}}kovi{\'c}.
\newblock Geometric deep learning: Grids, groups, graphs, geodesics, and
  gauges.
\newblock \emph{arXiv preprint arXiv:2104.13478}, 2021.

\bibitem[Bruintjes et~al.(2023)Bruintjes, Motyka, and van
  Gemert]{Bruintjes_2023_CVPR}
Robert-Jan Bruintjes, Tomasz Motyka, and Jan van Gemert.
\newblock What affects learned equivariance in deep image recognition models?
\newblock In \emph{Proceedings of the IEEE/CVF Conference on Computer Vision
  and Pattern Recognition (CVPR) Workshops}, pages 4838--4846, 2023.

\bibitem[Cavalli et~al.(2020)Cavalli, Larsson, Oswald, Sattler, and
  Pollefeys]{cavalli2020adalam}
Luca Cavalli, Viktor Larsson, Martin~Ralf Oswald, Torsten Sattler, and Marc
  Pollefeys.
\newblock Adalam: Revisiting handcrafted outlier detection.
\newblock \emph{arXiv preprint arXiv:2006.04250}, 2020.

\bibitem[Cohen and Welling(2014{\natexlab{a}})]{cohen14irreps}
Taco Cohen and Max Welling.
\newblock Learning the irreducible representations of commutative lie groups.
\newblock In \emph{Proceedings of the 31st International Conference on Machine
  Learning}, pages 1755--1763, Bejing, China, 2014{\natexlab{a}}. PMLR.

\bibitem[Cohen and Welling(2016)]{cohen2016group}
Taco Cohen and Max Welling.
\newblock Group equivariant convolutional networks.
\newblock In \emph{International conference on machine learning}, pages
  2990--2999. PMLR, 2016.

\bibitem[Cohen and Welling(2014{\natexlab{b}})]{cohen2014transformation}
Taco~S Cohen and Max Welling.
\newblock Transformation properties of learned visual representations.
\newblock \emph{ICLR 2015 (arXiv:1412.7659)}, 2014{\natexlab{b}}.

\bibitem[DeTone et~al.(2018)DeTone, Malisiewicz, and
  Rabinovich]{detone2018superpoint}
Daniel DeTone, Tomasz Malisiewicz, and Andrew Rabinovich.
\newblock Superpoint: Self-supervised interest point detection and description.
\newblock In \emph{Proceedings of the IEEE conference on computer vision and
  pattern recognition workshops}, pages 224--236, 2018.

\bibitem[Edixhoven et~al.(2023)Edixhoven, Lengyel, and van
  Gemert]{edixhoven2023using}
Tom Edixhoven, Attila Lengyel, and Jan~C van Gemert.
\newblock Using and abusing equivariance.
\newblock In \emph{Proceedings of the IEEE/CVF International Conference on
  Computer Vision}, pages 119--128, 2023.

\bibitem[Edstedt et~al.(2024)Edstedt, Bökman, Wadenbäck, and
  Felsberg]{edstedt2024dedode}
Johan Edstedt, Georg Bökman, Mårten Wadenbäck, and Michael Felsberg.
\newblock {DeDoDe: Detect, Don't Describe -- Describe, Don't Detect for Local
  Feature Matching}.
\newblock In \emph{2024 International Conference on 3D Vision (3DV)}. IEEE,
  2024.

\bibitem[Finzi et~al.(2021)Finzi, Welling, and Wilson]{finzi2021practical}
Marc Finzi, Max Welling, and Andrew~Gordon Wilson.
\newblock A practical method for constructing equivariant multilayer
  perceptrons for arbitrary matrix groups.
\newblock In \emph{International conference on machine learning}, pages
  3318--3328. PMLR, 2021.

\bibitem[Freeman et~al.(1991)Freeman, Adelson, et~al.]{freeman1991design}
William~T Freeman, Edward~H Adelson, et~al.
\newblock The design and use of steerable filters.
\newblock \emph{IEEE Transactions on Pattern analysis and machine
  intelligence}, 13\penalty0 (9):\penalty0 891--906, 1991.

\bibitem[Gerken et~al.()Gerken, Aronsson, Carlsson, Linander, Ohlsson,
  Petersson, and Persson]{gerkenGeometricDeepLearning2023}
Jan~E. Gerken, Jimmy Aronsson, Oscar Carlsson, Hampus Linander, Fredrik
  Ohlsson, Christoffer Petersson, and Daniel Persson.
\newblock Geometric deep learning and equivariant neural networks.
\newblock 56\penalty0 (12):\penalty0 14605--14662.

\bibitem[Gleize et~al.(2023)Gleize, Wang, and Feiszli]{gleize2023silk}
Pierre Gleize, Weiyao Wang, and Matt Feiszli.
\newblock {SiLK: Simple Learned Keypoints}.
\newblock In \emph{ICCV}, 2023.

\bibitem[Gruver et~al.(2023)Gruver, Finzi, Goldblum, and Wilson]{gruver2023the}
Nate Gruver, Marc~Anton Finzi, Micah Goldblum, and Andrew~Gordon Wilson.
\newblock The lie derivative for measuring learned equivariance.
\newblock In \emph{The Eleventh International Conference on Learning
  Representations}, 2023.

\bibitem[Gupta et~al.(2023)Gupta, Robinson, Lim, Villar, and
  Jegelka]{gupta2023structuring}
Sharut Gupta, Joshua Robinson, Derek Lim, Soledad Villar, and Stefanie Jegelka.
\newblock Structuring representation geometry with rotationally equivariant
  contrastive learning.
\newblock \emph{arXiv preprint arXiv:2306.13924}, 2023.

\bibitem[He et~al.(2015)He, Zhang, Ren, and Sun]{he2015delving}
Kaiming He, Xiangyu Zhang, Shaoqing Ren, and Jian Sun.
\newblock Delving deep into rectifiers: Surpassing human-level performance on
  imagenet classification.
\newblock In \emph{Proceedings of the IEEE international conference on computer
  vision}, pages 1026--1034, 2015.

\bibitem[Ke and Sukthankar(2004)]{ke2004pca}
Yan Ke and Rahul Sukthankar.
\newblock Pca-sift: A more distinctive representation for local image
  descriptors.
\newblock In \emph{Proceedings of the 2004 IEEE Computer Society Conference on
  Computer Vision and Pattern Recognition, 2004. CVPR 2004.}, pages II--II.
  IEEE, 2004.

\bibitem[Koyama et~al.(2023)Koyama, Fukumizu, Hayashi, and
  Miyato]{koyama2023nft}
Masanori Koyama, Kenji Fukumizu, Kohei Hayashi, and Takeru Miyato.
\newblock Neural fourier transform: A general approach to equivariant
  representation learning.
\newblock \emph{arXiv preprint arXiv:2305.18484}, 2023.

\bibitem[Lee et~al.(2021)Lee, Jeong, and Cho]{lee2021self}
Jongmin Lee, Yoonwoo Jeong, and Minsu Cho.
\newblock Self-supervised learning of image scale and orientation.
\newblock In \emph{31st British Machine Vision Conference 2021, {BMVC} 2021,
  Virtual Event, UK}. {BMVA} Press, 2021.

\bibitem[Lee et~al.(2022)Lee, Kim, and Cho]{lee2022self}
Jongmin Lee, Byungjin Kim, and Minsu Cho.
\newblock Self-supervised equivariant learning for oriented keypoint detection.
\newblock In \emph{Proceedings of the IEEE/CVF Conference on Computer Vision
  and Pattern Recognition}, pages 4847--4857, 2022.

\bibitem[Lee et~al.(2023)Lee, Kim, Kim, and Cho]{lee2023learning}
Jongmin Lee, Byungjin Kim, Seungwook Kim, and Minsu Cho.
\newblock Learning rotation-equivariant features for visual correspondence.
\newblock In \emph{Proceedings of the IEEE/CVF Conference on Computer Vision
  and Pattern Recognition}, pages 21887--21897, 2023.

\bibitem[Lenc and Vedaldi(2015)]{Lenc_2015_CVPR}
Karel Lenc and Andrea Vedaldi.
\newblock Understanding image representations by measuring their equivariance
  and equivalence.
\newblock In \emph{Proceedings of the IEEE Conference on Computer Vision and
  Pattern Recognition (CVPR)}, 2015.

\bibitem[Li and Snavely(2018)]{li2018megadepth}
Zhengqi Li and Noah Snavely.
\newblock Megadepth: Learning single-view depth prediction from internet
  photos.
\newblock In \emph{Proceedings of the IEEE Conference on Computer Vision and
  Pattern Recognition}, pages 2041--2050, 2018.

\bibitem[Liu et~al.(2019)Liu, Shen, Lin, Peng, Bao, and Zhou]{liu2019gift}
Yuan Liu, Zehong Shen, Zhixuan Lin, Sida Peng, Hujun Bao, and Xiaowei Zhou.
\newblock Gift: Learning transformation-invariant dense visual descriptors via
  group cnns.
\newblock \emph{Advances in Neural Information Processing Systems}, 32, 2019.

\bibitem[Lowe(2004)]{lowe2004distinctive}
David~G Lowe.
\newblock Distinctive image features from scale-invariant keypoints.
\newblock \emph{International journal of computer vision}, 60\penalty0
  (2):\penalty0 91--110, 2004.

\bibitem[Marchetti et~al.(2023)Marchetti, Tegn{\'e}r, Varava, and
  Kragic]{marchetti2023equivariant}
Giovanni~Luca Marchetti, Gustaf Tegn{\'e}r, Anastasiia Varava, and Danica
  Kragic.
\newblock Equivariant representation learning via class-pose decomposition.
\newblock In \emph{International Conference on Artificial Intelligence and
  Statistics}, pages 4745--4756. PMLR, 2023.

\bibitem[Mishkin et~al.(2018)Mishkin, Radenovic, and
  Matas]{mishkin2018repeatability}
Dmytro Mishkin, Filip Radenovic, and Jiri Matas.
\newblock Repeatability is not enough: Learning affine regions via
  discriminability.
\newblock In \emph{Proceedings of the European conference on computer vision
  (ECCV)}, pages 284--300, 2018.

\bibitem[Ono et~al.(2018)Ono, Trulls, Fua, and Yi]{ono2018lf}
Yuki Ono, Eduard Trulls, Pascal Fua, and Kwang~Moo Yi.
\newblock {LF-Net}: Learning local features from images.
\newblock \emph{Advances in neural information processing systems}, 31, 2018.

\bibitem[Oquab et~al.(2023)Oquab, Darcet, Moutakanni, Vo, Szafraniec, Khalidov,
  Fernandez, Haziza, Massa, El-Nouby, Howes, Huang, Xu, Sharma, Li, Galuba,
  Rabbat, Assran, Ballas, Synnaeve, Misra, Jegou, Mairal, Labatut, Joulin, and
  Bojanowski]{oquab2023dinov2}
Maxime Oquab, Timothée Darcet, Theo Moutakanni, Huy~V. Vo, Marc Szafraniec,
  Vasil Khalidov, Pierre Fernandez, Daniel Haziza, Francisco Massa, Alaaeldin
  El-Nouby, Russell Howes, Po-Yao Huang, Hu Xu, Vasu Sharma, Shang-Wen Li,
  Wojciech Galuba, Mike Rabbat, Mido Assran, Nicolas Ballas, Gabriel Synnaeve,
  Ishan Misra, Herve Jegou, Julien Mairal, Patrick Labatut, Armand Joulin, and
  Piotr Bojanowski.
\newblock {DINOv2}: Learning robust visual features without supervision.
\newblock \emph{arXiv:2304.07193}, 2023.

\bibitem[Paszke et~al.(2019)Paszke, Gross, Massa, Lerer, Bradbury, Chanan,
  Killeen, Lin, Gimelshein, Antiga, et~al.]{paszke2019pytorch}
Adam Paszke, Sam Gross, Francisco Massa, Adam Lerer, James Bradbury, Gregory
  Chanan, Trevor Killeen, Zeming Lin, Natalia Gimelshein, Luca Antiga, et~al.
\newblock Pytorch: An imperative style, high-performance deep learning library.
\newblock \emph{Advances in neural information processing systems}, 32, 2019.

\bibitem[Pautrat et~al.(2020)Pautrat, Larsson, Oswald, and
  Pollefeys]{pautrat2020lisrd}
Rémi Pautrat, Viktor Larsson, Martin~R. Oswald, and Marc Pollefeys.
\newblock Online invariance selection for local feature descriptors.
\newblock In \emph{Proceedings of the European Conference on Computer Vision
  (ECCV)}, 2020.

\bibitem[Pielawski et~al.(2020)Pielawski, Wetzer, \"{O}fverstedt, Lu,
  W\"{a}hlby, Lindblad, and Sladoje]{pielawski2020comir}
Nicolas Pielawski, Elisabeth Wetzer, Johan \"{O}fverstedt, Jiahao Lu, Carolina
  W\"{a}hlby, Joakim Lindblad, and Nata{\v{s}}a Sladoje.
\newblock {CoMIR}: Contrastive multimodal image representation for
  registration.
\newblock In \emph{Advances in Neural Information Processing Systems}, pages
  18433--18444. Curran Associates, Inc., 2020.

\bibitem[Revaud et~al.(2019)Revaud, De~Souza, Humenberger, and
  Weinzaepfel]{revaud2019r2d2}
Jerome Revaud, Cesar De~Souza, Martin Humenberger, and Philippe Weinzaepfel.
\newblock R2d2: Reliable and repeatable detector and descriptor.
\newblock \emph{Advances in neural information processing systems},
  32:\penalty0 12405--12415, 2019.

\bibitem[Rocco et~al.(2018)Rocco, Cimpoi, Arandjelovi\'c, Torii, Pajdla, and
  Sivic]{rocco2018neighbourhood}
I. Rocco, M. Cimpoi, R. Arandjelovi\'c, A. Torii, T. Pajdla, and J. Sivic.
\newblock Neighbourhood consensus networks.
\newblock In \emph{Proceedings of the 32nd Conference on Neural Information
  Processing Systems}, 2018.

\bibitem[Rublee et~al.(2011)Rublee, Rabaud, Konolige, and
  Bradski]{rublee2011orb}
Ethan Rublee, Vincent Rabaud, Kurt Konolige, and Gary Bradski.
\newblock {ORB: An efficient alternative to SIFT or SURF}.
\newblock In \emph{2011 International conference on computer vision}, pages
  2564--2571. Ieee, 2011.

\bibitem[Santellani et~al.(2023)Santellani, Sormann, Rossi, Kuhn, and
  Fraundorfer]{santellani2023s}
Emanuele Santellani, Christian Sormann, Mattia Rossi, Andreas Kuhn, and
  Friedrich Fraundorfer.
\newblock S-trek: Sequential translation and rotation equivariant keypoints for
  local feature extraction.
\newblock In \emph{Proceedings of the IEEE/CVF International Conference on
  Computer Vision}, pages 9728--9737, 2023.

\bibitem[Serre(1977)]{serreLinearRepresentationsFinite1977}
Jean-Pierre Serre.
\newblock \emph{Linear {{Representations}} of {{Finite Groups}}}.
\newblock {Springer}, 1977.

\bibitem[Shakerinava et~al.(2022)Shakerinava, Mondal, and
  Ravanbakhsh]{shakerinava2022structuring}
Mehran Shakerinava, Arnab~Kumar Mondal, and Siamak Ravanbakhsh.
\newblock Structuring representations using group invariants.
\newblock In \emph{Advances in Neural Information Processing Systems}, pages
  34162--34174. Curran Associates, Inc., 2022.

\bibitem[Stoken and Fisher(2023)]{stoken2023astronaut}
Alex Stoken and Kenton Fisher.
\newblock Find my astronaut photo: Automated localization and georectification
  of astronaut photography.
\newblock In \emph{Proceedings of the IEEE/CVF Conference on Computer Vision
  and Pattern Recognition (CVPR) Workshops}, pages 6196--6205, 2023.

\bibitem[Sun et~al.(2021)Sun, Shen, Wang, Bao, and Zhou]{sun2021loftr}
Jiaming Sun, Zehong Shen, Yuang Wang, Hujun Bao, and Xiaowei Zhou.
\newblock {LoFTR: Detector-free local feature matching with transformers}.
\newblock In \emph{Proceedings of the IEEE/CVF Conference on Computer Vision
  and Pattern Recognition}, pages 8922--8931, 2021.

\bibitem[Tao et~al.(2010)Tao, Vu, and Krishnapur]{tao2010random}
Terence Tao, Van Vu, and Manjunath Krishnapur.
\newblock {Random matrices: Universality of ESDs and the circular law}.
\newblock \emph{The Annals of Probability}, 38\penalty0 (5):\penalty0 2023 --
  2065, 2010.

\bibitem[Tian et~al.(2017)Tian, Fan, and Wu]{tian2017l2}
Yurun Tian, Bin Fan, and Fuchao Wu.
\newblock L2-net: Deep learning of discriminative patch descriptor in euclidean
  space.
\newblock In \emph{Proceedings of the IEEE conference on computer vision and
  pattern recognition}, pages 661--669, 2017.

\bibitem[Tian et~al.(2019)Tian, Yu, Fan, Wu, Heijnen, and
  Balntas]{tian2019sosnet}
Yurun Tian, Xin Yu, Bin Fan, Fuchao Wu, Huub Heijnen, and Vassileios Balntas.
\newblock Sosnet: Second order similarity regularization for local descriptor
  learning.
\newblock In \emph{Proceedings of the IEEE/CVF Conference on Computer Vision
  and Pattern Recognition}, pages 11016--11025, 2019.

\bibitem[Tian et~al.(2020)Tian, Barroso~Laguna, Ng, Balntas, and
  Mikolajczyk]{tian2020hynet}
Yurun Tian, Axel Barroso~Laguna, Tony Ng, Vassileios Balntas, and Krystian
  Mikolajczyk.
\newblock Hynet: Learning local descriptor with hybrid similarity measure and
  triplet loss.
\newblock \emph{Advances in neural information processing systems},
  33:\penalty0 7401--7412, 2020.

\bibitem[Tyszkiewicz et~al.(2020)Tyszkiewicz, Fua, and
  Trulls]{tyszkiewicz2020disk}
Michal~J. Tyszkiewicz, Pascal Fua, and Eduard Trulls.
\newblock {DISK:} learning local features with policy gradient.
\newblock In \emph{NeurIPS}, 2020.

\bibitem[Weiler and Cesa(2019)]{weiler2019general}
Maurice Weiler and Gabriele Cesa.
\newblock General e (2)-equivariant steerable cnns.
\newblock \emph{Advances in neural information processing systems}, 32, 2019.

\bibitem[Woit(2017)]{Woit2017}
Peter Woit.
\newblock \emph{Quantum {{Theory}}, {{Groups}} and {{Representations}}}.
\newblock {Springer International Publishing}, 2017.

\bibitem[Wood and Shawe-Taylor(1996)]{wood1996rep}
Jeffrey Wood and John Shawe-Taylor.
\newblock Representation theory and invariant neural networks.
\newblock \emph{Discrete Applied Mathematics}, 69\penalty0 (1):\penalty0
  33--60, 1996.

\bibitem[Worrall et~al.(2017{\natexlab{a}})Worrall, Garbin, Turmukhambetov, and
  Brostow]{Worall_2017_CVPR}
Daniel~E. Worrall, Stephan~J. Garbin, Daniyar Turmukhambetov, and Gabriel~J.
  Brostow.
\newblock Harmonic networks: Deep translation and rotation equivariance.
\newblock In \emph{Proceedings of the IEEE Conference on Computer Vision and
  Pattern Recognition (CVPR)}, 2017{\natexlab{a}}.

\bibitem[Worrall et~al.(2017{\natexlab{b}})Worrall, Garbin, Turmukhambetov, and
  Brostow]{Worrall_2017_ICCV}
Daniel~E. Worrall, Stephan~J. Garbin, Daniyar Turmukhambetov, and Gabriel~J.
  Brostow.
\newblock Interpretable transformations with encoder-decoder networks.
\newblock In \emph{Proceedings of the IEEE International Conference on Computer
  Vision (ICCV)}, 2017{\natexlab{b}}.

\bibitem[Yi et~al.(2016)Yi, Trulls, Lepetit, and Fua]{yi2016lift}
Kwang~Moo Yi, Eduard Trulls, Vincent Lepetit, and Pascal Fua.
\newblock Lift: Learned invariant feature transform.
\newblock In \emph{Computer Vision--ECCV 2016: 14th European Conference,
  Amsterdam, The Netherlands, October 11-14, 2016, Proceedings, Part VI 14},
  pages 467--483. Springer, 2016.

\end{thebibliography}
}

\clearpage
\clearpage
\setcounter{page}{1}
\maketitlesupplementary

\appendix
\section{Supplementary theory}
\label{app:theory}
We provide further theoretical discussions that did not have room in the main text.
First, Section~\ref{app:disentangle} contains a discussion of what having representations of $C_4$ or $\mathrm{SO}(2)$ on description space means.
Section~\ref{app:proof_gupta} contains a proof of Theorem~\ref{thm:gupta} and Section~\ref{app:matching_strat} presents matching strategies that are considered in the extra ablations of Section~\ref{app:more_experiments} but were omitted from the main paper due to space limitations.

\subsection{Disentangling description space}
\label{app:disentangle}
As explained following Theorem~\ref{thm:c4-rep}, any representation of $C_4$ or $\mathrm{SO}(2)$ can be block-diagonalized over the real numbers into blocks of size $1$ and $2$, called irreducible representations (irreps).
We can think of these irreps as disentangling
descriptions space \cite{cohen14irreps}, i.e.
each eigenspace of the steerer
is acted on by rotations in a specific way according to
the respective irrep.
This section explains the relevance of the different irreps to keypoint descriptors.
For $C_4$, we have the following.
\begin{itemize}
    \item $1\times 1$ irreps $\begin{pmatrix}1\end{pmatrix}$ act by doing nothing.
        Hence, the corresponding dimensions in description space are invariant under rotations.
        For $C_4$, image features described by these dimensions could be crosses or blobs.
    \item $1\times  1$ irreps $\begin{pmatrix}-1\end{pmatrix}$ correspond to dimensions that are invariant under $\ang{180}$ rotations, but not $\ang{90}$ rotations.
        Examples of such image features could be lines.
    \item $2\times 2$ irreps $\imagblock$ correspond to pairs of dimensions that are not invariant under any rotation. Many image features should be of this type, \eg corners.
\end{itemize}
For $\mathrm{SO}(2)$ we get the same $\begin{pmatrix}1\end{pmatrix}$ irrep, which in this case represents features invariant under all rotations such as blobs, and the $2\times 2$ irreps in \eqref{eq:so2_blocks} which represent features rotating with
$j$ times the frequency of the image. E.g. lines rotate with frequency $j=2$ since when we rotate the image by $\ang{180}$, the line returns to its original orientation.

The description for a keypoint does not lie solely in the dimensions of a single irrep. It will be a linear combination of quantities that transform according to the different irreps.
The descriptions can then be viewed as a form of non-linear Fourier decomposition of the image features, as discussed in the literature for general image features. We will provide a short discussion in the next paragraph.
\begin{exmp}
    In Upright SIFT, the decomposition of the $128$ description dimensions is equally split between the irreps, \ie, there are $32$ invariant dimensions, $32$ dimensions that are invariant under $\ang{180}$ degree rotations and $64$ dimensions which are not invariant under any rotation.
\end{exmp}

The connections to Fourier analysis of having a group representation acting on the 
latent space of a model
were discussed in \cite{cohen14irreps} for linear models and
concurrently to this work for neural networks in \cite{koyama2023nft}.
We sketch the main idea here to give the reader some intuition.
If we have a signal $h$ on $\mathbb{R}^n$ and want to know how it transforms under cyclic permutations of the $n$ coordinates, we can take the Discrete Fourier Transform (DFT).
Each coordinate $h_j$ can be written as a linear combination of Fourier basis functions:
$h_j = \sum_{k=0}^{n - 1} \hat h_k \exp(2 \pi \mathbf{i} jk / n)$
and the DFT is simply the vector $\hat h$.
When we cyclically permute $h$ by $J$ steps,
it corresponds to multiplying each component $\hat h_k$ by $\exp(2\pi\mathbf{i}Jk/n)$.
Thus, the cyclic permutation on $h$ acts like a diagonal matrix on the DFT $\hat h$.
The DFT is a linear transform of the signal $h$.
In our setting, the signal consists of images and keypoints transformed by a neural network $f$ to description space.
As described in Theorem~\ref{thm:c4-rep}, 
rotations act by a diagonal matrix in description space (up to a change of basis $Q$).
In the terminology of \cite{koyama2023nft}, we can think of the neural network $f$ as doing a Neural Fourier Transform of the input.

The usefulness of having group representations act on latent spaces in neural networks has been considered in, for instance, \cite{Worrall_2017_ICCV, cohen2014transformation, marchetti2023equivariant}.
In these works, the specific representation is fixed before training the network, similar to our Setting C.
As far as we know, optimally choosing the representation remains an open question---the experiments in this paper showed that this is an important question.
\cite{shakerinava2022structuring} and \cite{gupta2023structuring} considered using \eqref{eq:gupta} as a loss term to obtain orthogonal representations on the latent space.
This approach is also promising for keypoint descriptors, particularly for encoding transformations more complicated than rotations in description space, since it does not require knowledge of the representation theory of the transformation group in question.

\subsection{Proof of Theorem~\ref{thm:gupta}}
\label{app:proof_gupta}
Here, we give a proof of Theorem~\ref{thm:gupta}.
\thmgupta*
\begin{proof}
    Note that if $f(v)=f(v')$, then
    by \eqref{eq:gupta}, $f(\rho_\text{in}(g)v)=f(\rho_\text{in}(g)v')$.
    This means that we, for each $g\in G$, can define a map
    $\tilde\varphi_g:f(V)\to f(V)$ by 
    \begin{equation}\label{eq:tildephi}
        \tilde \varphi_g(w)=f(\rho_\text{in}(g)f^{-1}(w)),
    \end{equation} where $f^{-1}(w)$ is any element that $f$ maps to $w$.
    We next extend $\tilde\varphi_g$ to a map $\varphi_g:\mathbb{S}^{D-1}\to\mathbb{S}^{D-1}$.
    Start by writing any element $w\in\mathbb{S}^{D-1}$ as 
    \begin{equation}
        w=w_\bot + \sum_{i=1}^n a_i w_i,
    \end{equation}
    where $a_i\in\mathbb{R}$, the $w_i$'s are of the form $f(v_i)$ for some $v_i\in V$ and form a basis of $\mathrm{span}(f(V))$ and $w_\bot$ is orthogonal to $\mathrm{span}(f(V))$.
    Define 
    \begin{equation}\label{eq:defphi}
        \varphi_g(w) = w_\bot + \sum_{i=1}^n a_i \tilde\varphi_g(w_i).
    \end{equation}
    We can now use \eqref{eq:gupta} to show that $\varphi_g$ is an isometry of the sphere $\mathbb{S}^{D-1}$, i.e. an orthogonal transformation:
    \begin{align}
        \langle\varphi_g(w), \varphi_g(w')\rangle &= \langle w_\bot, w'_\bot \rangle \\
        &\phantom{=} + \sum_{i=1}^n\sum_{j=1}^n a_i a'_j \langle \tilde\varphi_g(w_i), \tilde\varphi_g(w_j) \rangle \\
        &\stackrel{\substack{\eqref{eq:gupta}\\ \eqref{eq:tildephi}}}{=} \langle w_\bot, w'_\bot \rangle + \sum_{i=1}^n\sum_{j=1}^n a_i a'_j \langle w_i, w_j \rangle \\
        &= \langle w, w' \rangle.
    \end{align}
    As it is an orthogonal transformation, we can write $\varphi_g$ as being a matrix acting on vectors in $\mathbb{S}^{D-1}$ by matrix multiplication.
    Finally, we need to show that $\rho(g) = \varphi_g$ defines a representation of $G$, i.e. that $\varphi_g\varphi_{g'} = \varphi_{gg'}$ for all $g,g'\in G$.
    We begin by showing that $\varphi_g$ and $\tilde\varphi_g$ are
    equal on $f(V)$, which now follows from linearity of $\varphi_g$ as follows.
    Take a general $w\in f(V)$,
    and again write $w=\sum_{i=1}^n a_i w_i$, then
    \begin{align}
        \langle \varphi_g(w), \tilde\varphi_g(w) \rangle 
        &= \left\langle \varphi_g\left(\sum_{i=1}^n a_iw_i\right), \tilde\varphi_g(w)\right\rangle \\
        &= \sum_{i=1}^n a_i \langle \varphi_g(w_i), \tilde\varphi_g(w)\rangle \\
        &\stackrel{\eqref{eq:defphi}}{=} \sum_{i=1}^n a_i \langle \tilde\varphi_g(w_i), \tilde\varphi_g(w)\rangle \\
        &\stackrel{\eqref{eq:gupta}}{=} \sum_{i=1}^n a_i \langle w_i, w \rangle \\
        &= \langle w, w \rangle \\
        &= 1
    \end{align}
    so that $\varphi_g(w) = \tilde\varphi_g(w)$.
    For the $w_i$'s from before, we thus have 
    \begin{align}
        \varphi_{gg'}w_i &= \tilde\varphi_{gg'}(w_i) \\
        &= f(\rho_\text{in}(gg')v_i) \\
        &= f(\rho_\text{in}(g)\rho_\text{in}(g')v_i) \\
        &= \tilde\varphi_g (f(\rho_\text{in}(g')v_i)) \\
        &= \tilde\varphi_g (\tilde\varphi_{g'} (f(v_i))) \\
        &= \tilde\varphi_g (\tilde\varphi_{g'} (w_i)) \\
        &= \tilde\varphi_g (\varphi_{g'}(w_i)) \\
        &= \varphi_g \varphi_{g'} w_i.
    \end{align}
    Further, for any $w_\bot$ orthogonal to $\mathrm{span}(f(V))$ we have
    \begin{equation}
        \varphi_{gg'}w_\bot = w_\bot = \varphi_g\varphi_{g'} w_\bot.
    \end{equation}
    Thus by linearity $\varphi_{gg'} = \varphi_{g}\varphi_{g'}$.
\end{proof}

\subsection{More matching strategies}
\label{app:matching_strat}
We discuss more potential matching strategies.
Their performance is shown in the large ablation Table~\ref{tab:megadepth_large}.

\begin{enumerate}[font=\bfseries, wide, labelwidth=!, labelindent=0pt]
\item[Projecting to the invariant subspace.]
Given a steerer $\rho(\mathbf{g})$, we can project to the rotation invariant subspace of
the descriptions by taking $\sum_{k=0}^3\rho(\mathbf{g})^k y / 4$ as descriptions instead of $y$.
Equivalently, one can project by decomposing $\rho(\mathbf{g})$ using \eqref{eq:eigdecom_generator}.
However, we will see that these invariant descriptions do not perform very well (but still better than just using $y$).
This is likely because the invariant subspace is typically only a fourth of the
descriptor space.
\item[Subset matcher.] We can estimate the best relative rotation between two images using the \emph{max matches} matching strategy on only a subset of the keypoints in each image.
The obtained rotation is then used to steer the descriptions of all keypoints.
This \emph{subset matcher} strategy gives lower runtime while not sacrificing performance much.
In our experiments, we use $1,000$ keypoints.
\item[Prototype Procrustes matcher.]
For frequency 1 descriptors, as a way to make the Procrustes matcher less computationally expensive,
we propose, instead of aligning each description pair optimally,
to align every description to a prototype description $\tilde y\in\mathbb{R}^{2\times (D/2)}$.
Thus, we solve the Procrustes problem once per description in each image
to obtain $2N$ rotation matrices $R_{1,m}$ and $R_{2,n}$ and form the matching matrix
with elements $\langle\mathtt{flatten}(R_{1, m}y_{1,m}), \mathtt{flatten}(R_{2,n}y_{2,n})\rangle$.
$\tilde y$ can be obtained by optimizing it over a subset of the training set for a fixed frequency 1 descriptor. 
This strategy is similar to the group alignment proposed in RELF~\cite{lee2023learning}; however, there, the alignment is done using a single feature in a permutation representation of $C_{16}$, whereas we look at the entire $D$-dimensional description.
Similarly to RELF, we could use only specific dimensions of the descriptions for alignment and add a loss for this during training. 
However, we leave this and a careful examination of optimal alignment strategies for future work.
\end{enumerate}

\section{More experiments}
\label{app:more_experiments}
\begin{figure*}
    \centering
    \includegraphics[width=0.9\linewidth]{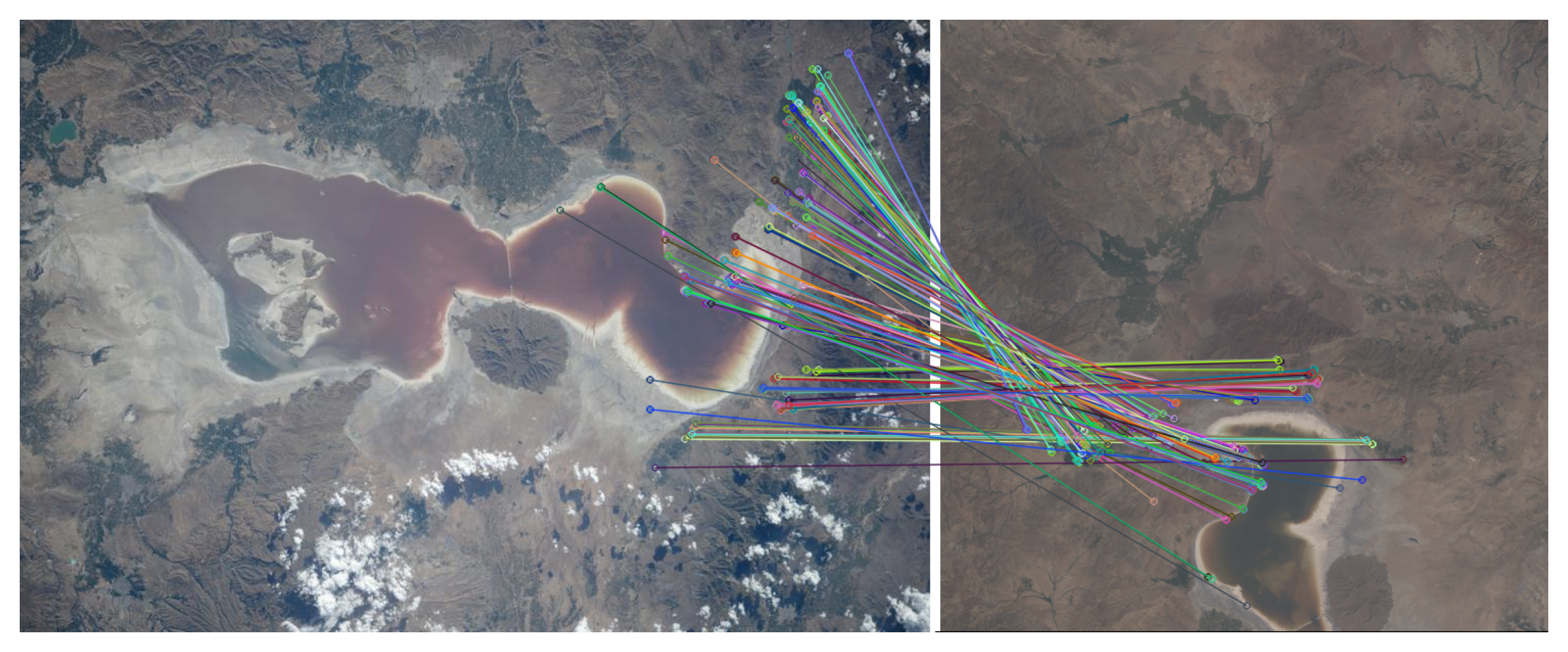}
    \includegraphics[width=0.9\linewidth]{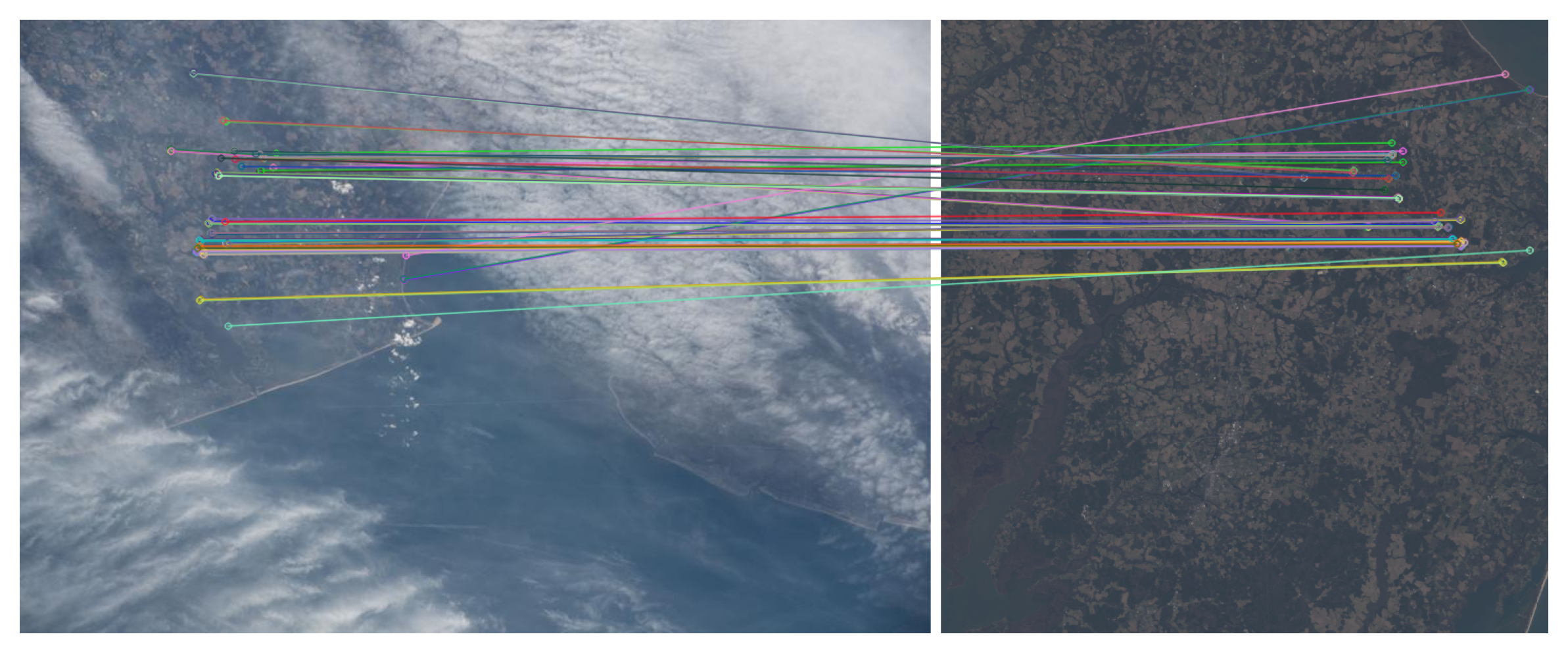}
    \includegraphics[width=0.9\linewidth]{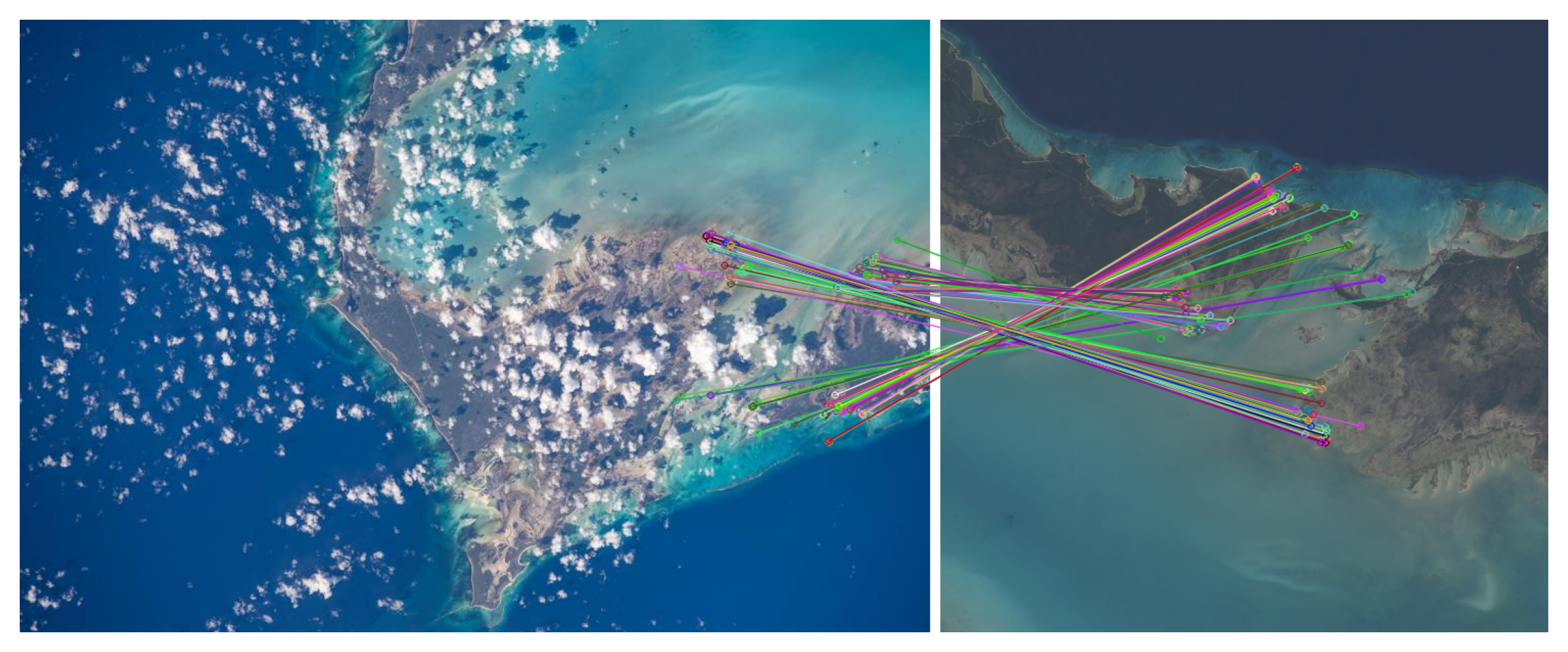}
    \caption{More qualitative challenging matching examples from the AIMS data.}
    \label{fig:aims-more}
\end{figure*}
\begin{figure*}
    \centering
    \includegraphics[width=0.9\linewidth]{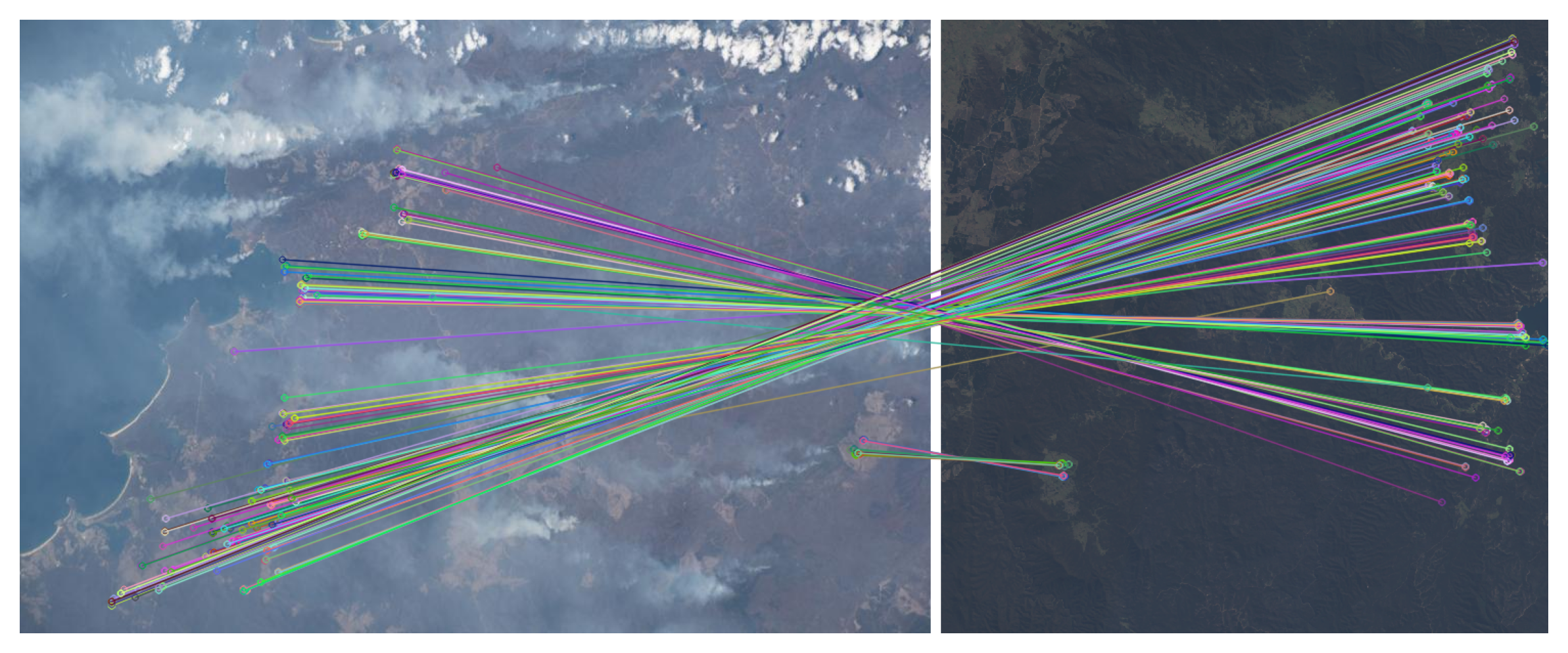}
    \includegraphics[width=0.9\linewidth]{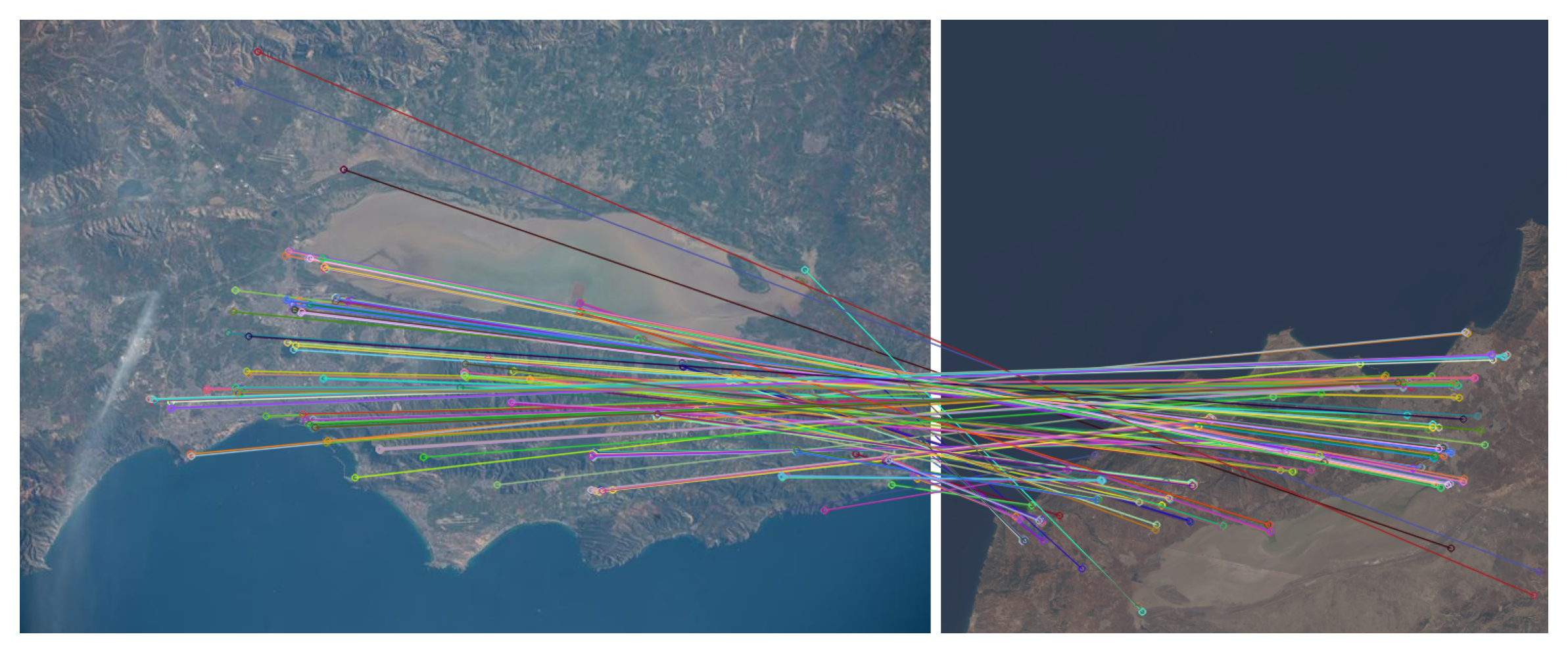}
    \includegraphics[width=0.9\linewidth]{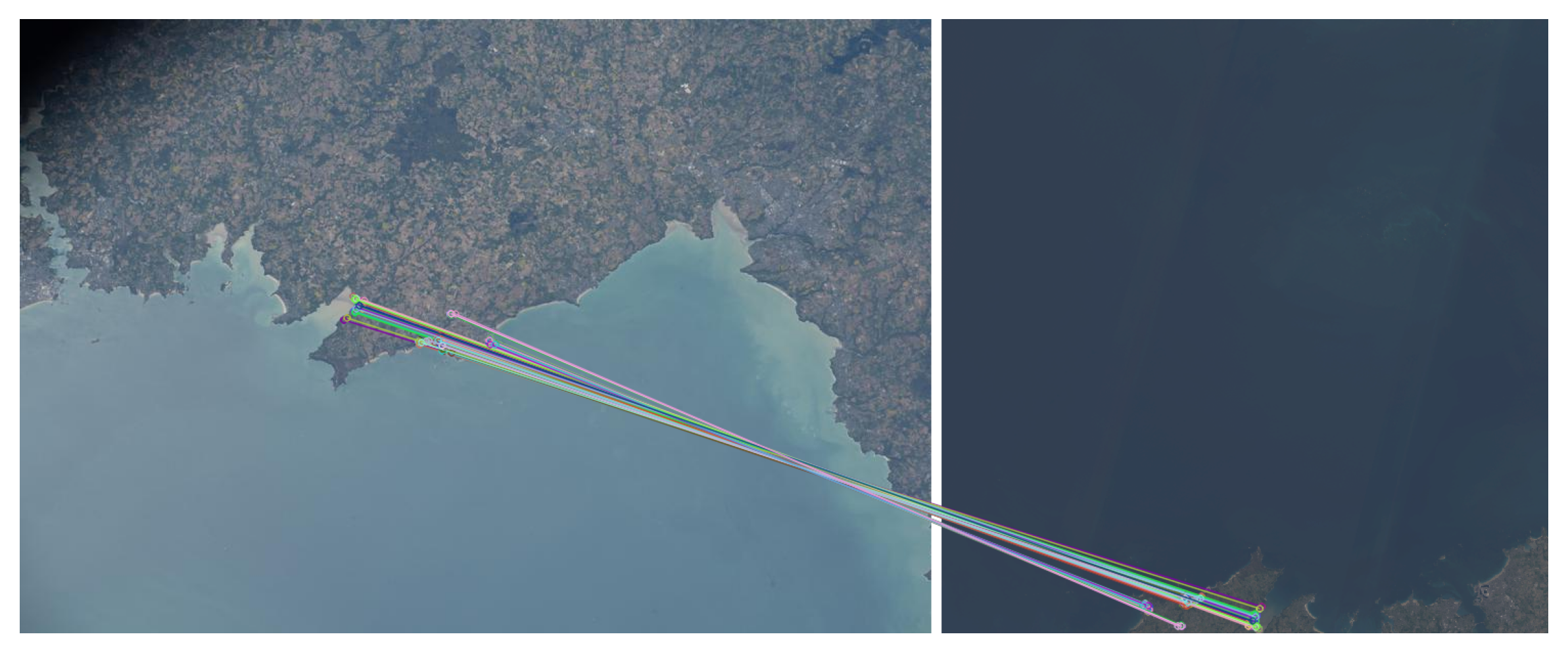}
    \caption{More qualitative challenging matching examples from the AIMS data.}
    \label{fig:aims-more2}
\end{figure*}
We show further matching examples on AIMS in Figures~\ref{fig:aims-more} and \ref{fig:aims-more2}.
These examples were chosen by selecting pairs with 20-200 matches after RANSAC, corresponding to successful but challenging pairs.
In the remainder of this section, we present ablations that did not have room in the main paper.
A large results table is provided as Table~\ref{tab:megadepth_large}.
Further, we present an experiment in support of using Theorem~\ref{thm:gupta} to motivate the existence of steerers in Section~\ref{app:equal_rot}
and a comparison to test time rotation augmentation in terms of performance and runtime in Section~\ref{app:tta}

Figure~\ref{fig:steer-rot2} shows an example of the improvement obtained using a steerer for large rotations.
Figure~\ref{fig:aims-recall-precision} shows the recall-precision curves for the experiments on AIMS from Section~\ref{sec:aims}.

\subsection{Equal rotation augmentation}
\label{app:equal_rot}
We explained the spontaneous equivariance of DeDoDe-B by referring to Theorem~\ref{thm:gupta} and saying that descriptors will be equivariant if the performance is equivalent for matching images $I_1$ and $I_2$ as matching jointly rotated images $P_{90}^kI_1$ and $P_{90}^kI_2$.
To test this explanation, we can look at whether the equivariance of a
keypoint descriptor relates to how good it is at matching jointly rotated images.

We experiment with four different descriptors, DeDoDe-B, DeDoDe-G, DISK~\cite{tyszkiewicz2020disk} and a retrained DeDoDe-B with data augmentation where both images are rotated an equal multiple of $\ang{90}$.
This retrained version is denoted DeDoDe-B$^\dag$.
The results are shown in Table~\ref{tab:joint_rot} and show that DeDoDe-B and DISK, for which the dropoff in performance between upright and jointly rotated images is relatively low, the steered performance is relatively high.
Conversely, DeDoDe-G has a large dropoff in performance between upright and jointly rotated images, and it also has a worse-performing steerer.
Finally, the retrained DeDoDe-B$^\dag$, trained to perform well on jointly rotated images, has a more or less perfect steerer.

\subsection{Comparison to test time augmentation}
\label{app:tta}
Given two images with an unknown relative rotation, the best obtainable matches from test time augmentation would be obtained when rotating the first image to have the same rotation from upright as the second, which is the case in the joint rotation benchmark.
The joint rotation benchmark considered in the previous section hence gives an upper bound for how well test time augmentation can work.
We also include the results of using C4-TTA with ordinary DeDoDe-B in Table~\ref{tab:megadepth_large}.
Therefore, Tables~\ref{tab:joint_rot} and \ref{tab:megadepth_large} show that using test time augmentation can give higher performance than a steerer in Setting A (Section~\ref{sec:settings}) of optimizing a steerer given a fixed descriptor.
The steerers obtained in Settings B and C, however, clearly outperform test time augmentation for the original DeDoDe networks (compare Table~\ref{tab:megadepth_large} and Table~\ref{tab:joint_rot}).
Table~\ref{tab:runtime} presents the improved runtime of using steerers.

\section{Experimental details}
\label{app:experiment_details}
We use the publicly available training code from DeDoDe~\cite{edstedt2024dedode} to train our models.
In Setting~A, we train the steerer for 10k iterations with a learning rate of $0.01$.
In Setting~B, we set the learning rate of the steerer to $2\cdot 10^{-4}$, which is the same as for the decoder in \cite{edstedt2024dedode} and train for 100k iterations as in \cite{edstedt2024dedode}.
In Setting~C, we also train for 100k iterations.
All other hyperparameters are identical to \cite{edstedt2024dedode}.

\subsection{How to initialize/fix the steerer}
Theorem~\ref{thm:gupta} tells us that the representation acting
on the description space (\ie the steerer) should be orthogonal.
Further, since we match using cosine similarity, we can perform an orthogonal change of basis in description space without influencing matching.
Thus, using the representation theory described in Section~\ref{sec:equiv_steer}, we can always change the basis of description space so that the steerer is block-diagonal with blocks of size 1 and 2.
Next, we describe the exact forms of steerers in our experiments when using different initializations or fixed steerers.
The labels correspond to the ones described in Section~\ref{sec:models_considered}.
Again we have two different cases depending on whether we have a $C_4$ steerer $\rho(\mathbf{g})$ or a $\mathrm{SO}(2)$ steerer obtained from a Lie algebra generator $\mathrm{d}\varsigma$.
\begin{enumerate}[font=\bfseries, wide, labelwidth=!, labelindent=0pt]
    \item[Inv.] Here, the steerer is simply the identity matrix.
    \item[Freq1.] We set the steerer $\rho(\mathbf{g})$ or the Lie algebra generator $\mathrm{d}\varsigma$ to 
        \begin{equation}
            \bigoplus_{b=1}^{128} \imagblock.
        \end{equation}
        Each block has eigenvalues $\pm\mathbf{i}$, so we get $128$ of each.
    \item[Perm.] We set the steerer $\rho(\mathbf{g})$ to 
        \begin{equation}
            \bigoplus_{b=1}^{64} \begin{pmatrix}
                0 & 1 & 0 & 0 \\
                0 & 0 & 1 & 0 \\
                0 & 0 & 0 & 1 \\
                1 & 0 & 0 & 0
            \end{pmatrix}
        \end{equation}
        Each block has eigenvalues $\pm 1, \pm \mathbf{i}$, so we get $64$ of each.
    \item[Spread.] We set the Lie algebra generator $\mathrm{d}\varsigma$ to
        \begin{equation}
            \left(
            \bigoplus_{b=1}^{40} \begin{pmatrix}
                0 
            \end{pmatrix}
            \right)
            \oplus
            \left(
            \bigoplus_{j=1}^{6}
            \left(
            \bigoplus_{b=1}^{18} \begin{pmatrix}
                0 & -j \\
                j & 0
            \end{pmatrix}
            \right)
            \right)
        \end{equation}
        Here, the first $40\times 40$ zero matrix gives $40$ eigenvalues $0$, the remaining blocks give $18$ eigenvalues of each $\pm j\mathbf{i}$ for $j=1,2,3,4,5,6$.
\end{enumerate}

\subsection{AIMS details}
In contrast to \cite{stoken2023astronaut}, we compute the average precision over the entire precision-recall curve instead of at fixed thresholds of the number of inliers.
The thresholds used in \cite{stoken2023astronaut} were chosen to approximately cover the precision-recall curve for the methods considered there. Still, we find using the complete precision-recall curve easier than rescaling the thresholds for our methods.
Furthermore, in \cite{stoken2023astronaut}, the average precision was computed with a maximum of 100 negative satellite images per astronaut photo.
Instead, for each astronaut photo, we use all associated satellite images in the AIMS to get more accurate precision scores.
These changes were agreed upon with the authors of \cite{stoken2023astronaut}.

We use only the Scale-1 subset of AIMS as we aim to evaluate rotational robustness.
Following \cite{stoken2023astronaut}, we resize all images so that the smallest side is 576px during matching.
For homography estimation, we use OpenCV with flag \texttt{USAC\_MAGSAC} \cite{barath2020magsac++} with confidence $0.999$, max iterations $10,000$ and inlier threshold $5$ pixels.
These settings were given to us by the authors of \cite{stoken2023astronaut}.
We rerun SE2-LoFTR to compute the average precision as described above. We confirmed that we get approximately the same score using the old evaluation protocol for SE2-LoFTR as reported in \cite{stoken2023astronaut} (they report $0.62$ on Upright and $0.51$ on All Others, while we get $0.60$ and $0.52$ respectively).

\newpage
\begin{figure}
        \centering
    \includegraphics[width=0.9\columnwidth]{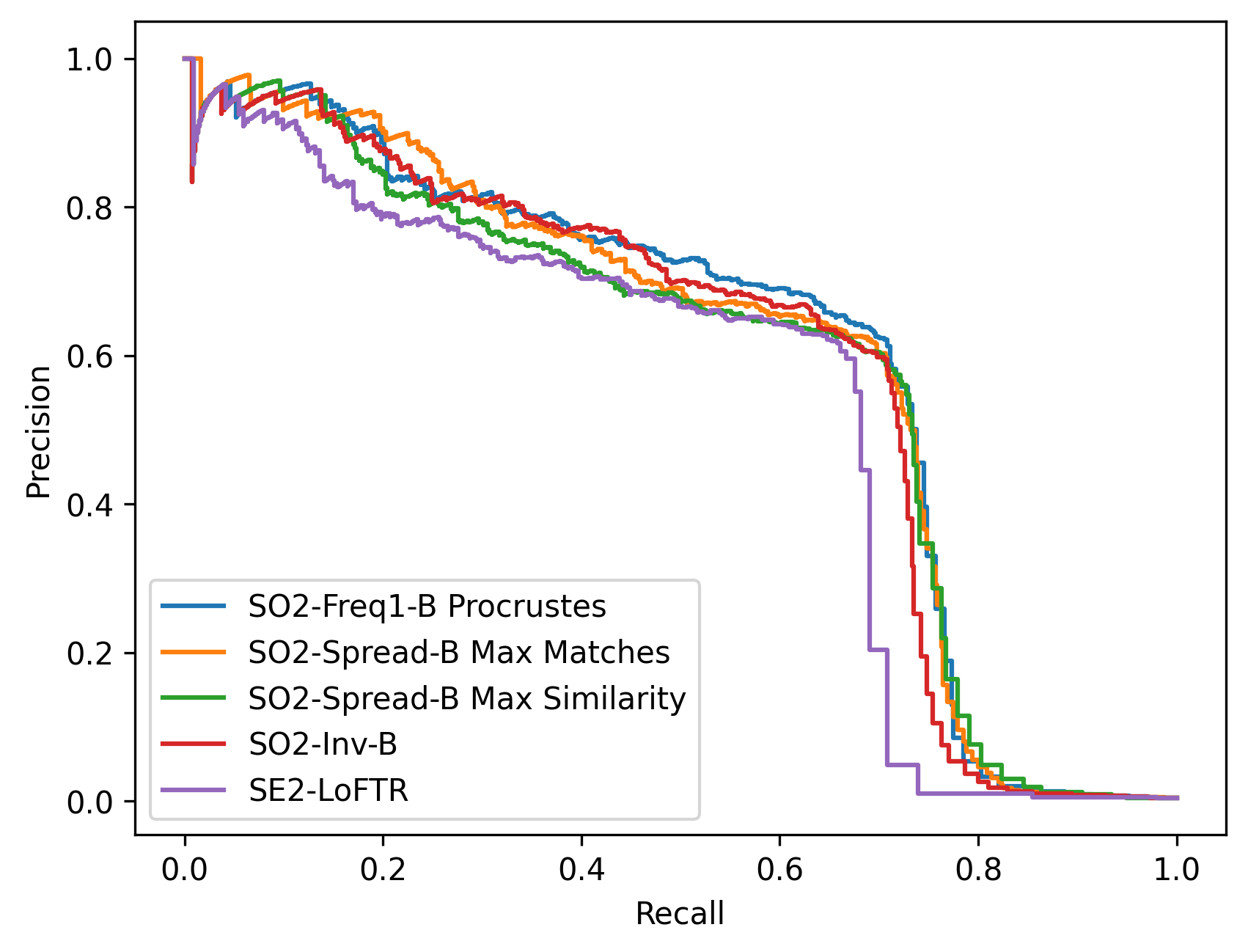}
    \caption{\textbf{Precision-recall on AIMS.} We plot precision-recall curves over the complete AIMS.}
    \label{fig:aims-recall-precision}
\end{figure}
\begin{table}
 \small
     \centering
     \caption{\textbf{Runtime comparison}. 
     We report the mean runtime over 100 random image tensors for description and matching on a single A100 GPU. I.e. the time for loading images and detection of keypoints is not measured.
     We use resolution $784\times 784$ and $5,000$ keypoints throughout.}
     \begin{tabular}{
        ll
        r
     }

     \toprule
      Descriptor & Matching strategy & \multicolumn{1}{l}{Time [ms]}
      \\
    
    \midrule

    DeDoDe-B & Dual softmax &  $63.4 \pm 0.03$ \\
    DeDoDe-B & Dual softmax + TTAx4 & $254.2 \pm 0.10$ \\
    DeDoDe-B & Dual softmax + TTAx8 & $513.0 \pm0.09$ \\
    DeDoDe-B-C4 & Max matches C4-steered & $96.5 \pm0.05$ \\
    DeDoDe-B-C4 & Subset C4-steered & $68.4 \pm0.10$ \\
    DeDoDe-B-C4 & Max similarity C4-steered & $66.9 \pm0.02$ \\
    DeDoDe-B-SO2 & Max matches C8-steered & $141.4 \pm0.05$ \\
    DeDoDe-B-SO2 & Subset C8-steered & $73.4 \pm0.13$ \\
    DeDoDe-B-SO2 & Max similarity C8-steered & $72.4 \pm0.02$ \\
    DeDoDe-B-SO2 & Procrustes & $95.0 \pm0.04$ \\
    DeDoDe-B-SO2 & Prototype Procrustes & $63.9 \pm0.02$ \\
    \midrule
    DeDoDe-G & Dual softmax & $217.3 \pm0.11$ \\
    DeDoDe-G & Dual softmax + TTAx4 & $872.5 \pm0.11$ \\
    DeDoDe-G-C4 & Max matches C4-steered & $250.4 \pm0.09$ \\
    DeDoDe-G-C4 & Subset C4-steered & $222.9 \pm0.08$ \\
    DeDoDe-G-C4 & Max similarity C4-steered & $221.3 \pm0.05$ \\
     \bottomrule
     \end{tabular}
     \label{tab:runtime}
\end{table}
\begin{table*}
 \small
     \centering
     \caption{\textbf{Evaluation on MegaDepth extended}. 
     This is a larger version of Table~\ref{tab:megadepth}.
     The first section shows Setting A where we only optimize the steerer, the second section shows Setting~B where we jointly optimize the descriptor and steerer and the third section shows Setting~C where we predefine the steerer and optimize only the descriptor.
     For MegaDepth-1500 we always use dual softmax matcher to evaluate the descriptors on upright images, except when the matching strategy is marked by $^*$, in which case we use the specified matching strategy for MegaDepth-1500 as well.
     We use $20,000$ keypoints throughout. The best values for \textcolor{Blue}{B}- and \textcolor{OliveGreen}{G}-models are highlighed in each column.  See Section~\ref{sec:models_considered} for shorthand explanations for our models.}
     \begin{tabular}{
        lll
        rrr 
        rrr
        rrr
     }
     \toprule
      Detector & Descriptor & & \multicolumn{3}{l}{\phantom{AUC @ }MegaDepth-1500} &  \multicolumn{3}{l}{MegaDepth-C4} &  \multicolumn{3}{l}{MegaDepth-SO2} \\ 
      DeDoDe & DeDoDe & Matching strategy & AUC $@$
      ~$5^{\circ}$&$10^{\circ}$&$20^{\circ}$ & 
      ~$5^{\circ}$&$10^{\circ}$&$20^{\circ}$ & 
      ~$5^{\circ}$&$10^{\circ}$&$20^{\circ}$\\
    \midrule
         Original & B & Dual softmax & 
            49 & 65 & 77 &
            12 & 17 & 20 &
            12 & 16 & 20
        \\
         Original & B & Dual softmax + TTA C4 & 
            \dittotikz & \dittotikz & \dittotikz &
            46 & 61 & 73 &
            34 & 49 & 61
        \\
         Original & B & Max matches C4-steered & 
            \dittotikz & \dittotikz & \dittotikz & 
            43 & 60 & 73 &
            30 & 44 & 56 
        \\
        C4 & B & Max matches C4-steered & 
            50 & 66 & 78 &
            43 & 60 & 74 &
            30 & 44 & 56
        \\
        C4 & B & Project to invariant subspace$^*$ &
            39 & 55 & 68 &
            33 & 49 & 62 &
            18 & 31 & 45
        \\
        SO2 & B & Max matches C4-steered & 
            50 & 66 & 78 &
            44 & 61 & 74 &
            30 & 45 & 58
        \\
        SO2 & B & Max matches C8-steered & 
            50 & 66 & 78 &
            40 & 57 & 70 &
            34 & 51 & 65
        \\
         Original & G & Dual softmax & 
            \textbf{\textcolor{OliveGreen}{52}} & \textbf{\textcolor{OliveGreen}{69}} & \textbf{\textcolor{OliveGreen}{81}} &
            13 & 17 & 21 &
            16 & 22 & 28
        \\
         Original & G & Max matches C4-steered & 
            \dittotikz & \dittotikz & \dittotikz & 
            31 & 45 & 57 &
            26 & 39 & 50 
        \\
    \midrule
        C4 & C4-B & Max matches C4-steered &
            \textbf{\textcolor{Blue}{51}} & \textbf{\textcolor{Blue}{67}} & \textbf{\textcolor{Blue}{79}} &
            \textbf{\textcolor{Blue}{50}} & \textbf{\textcolor{Blue}{67}} & \textbf{\textcolor{Blue}{79}} &
            39 & 55 & 68 
        \\
        C4 & C4-B & Subset C4-steered &
            \dittotikz & \dittotikz & \dittotikz &
            50 & 66 & 78 &
            39 & 54 & 68 
        \\
        C4 & C4-B & Max similarity C4-steered &
            50 & \textbf{\textcolor{Blue}{67}} & \textbf{\textcolor{Blue}{79}} &
            49 & 65 & 78 &
            35 & 50 & 62 
        \\
        SO2 & SO2-B & Max matches C8-steered &
            47 & 63 & 76 &
            47 & 63 & 76 &
            44 & 61 & 74 
        \\
        SO2 & SO2-Spread-B & Max matches C8-steered &
            50 & 66 & \textbf{\textcolor{Blue}{79}} &
            49 & 66 & 78 &
            \textbf{\textcolor{Blue}{46}} & \textbf{\textcolor{Blue}{63}} & \textbf{\textcolor{Blue}{76}} 
        \\
        SO2 & SO2-Spread-B & Subset C8-steered &
            \dittotikz & \dittotikz & \dittotikz &
            49 & 65 & 78 &
            \textbf{\textcolor{Blue}{46}} & 62 & 75 
        \\
        SO2 & SO2-Spread-B & Max similarity C8-steered &
            49 & 66 & 78 &
            47 & 64 & 77 &
            43 & 61 & 74 
        \\
    \midrule
        C4 & C4-Inv-B & Dual softmax &
            48 & 64 & 76 &
            47 & 63 & 76 &
            39 & 55 & 69 
        \\
        C4 & C4-Perm-B & Max matches C4-steered &
            50 & \textbf{\textcolor{Blue}{67}} & \textbf{\textcolor{Blue}{79}} &
            \textbf{\textcolor{Blue}{50}} & 66 & \textbf{\textcolor{Blue}{79}} &
            39 & 54 & 67 
        \\
        C4 & C4-Freq1-B & Max matches C4-steered &
            49 & 66 & 78 &
            49 & 65 & 78 &
            36 & 51 & 64 
        \\
        SO2 & SO2-Inv-B & Dual softmax &
            46 & 62 & 75 &
            45 & 61 & 74 &
            43 & 60 & 73 
        \\
        SO2 & SO2-Freq1-B & Max matches C8-steered &
            47 & 64 & 77 &
            47 & 64 & 76 &
            45 & 62 & 75 
        \\
        SO2 & SO2-Freq1-B & Procrustes &
            47 & 64 & 76 &
            46 & 62 & 75 &
            45 & 61 & 74 
        \\
        SO2 & SO2-Freq1-B & Prototype Procrustes &
            44 & 61 & 74 &
            43 & 60 & 73 &
            41 & 58 & 72 
        \\
        C4 & C4-Perm-G & Max matches C4-steered &
            \textbf{\textcolor{OliveGreen}{52}} & \textbf{\textcolor{OliveGreen}{69}} & \textbf{\textcolor{OliveGreen}{81}} &
            \textbf{\textcolor{OliveGreen}{53}} & \textbf{\textcolor{OliveGreen}{69}} & \textbf{\textcolor{OliveGreen}{82}} &
            \textbf{\textcolor{OliveGreen}{44}} & \textbf{\textcolor{OliveGreen}{61}} & \textbf{\textcolor{OliveGreen}{74}} 
        \\
        C4 & C4-Perm-G & Subset C4-steered &
            \dittotikz & \dittotikz & \dittotikz &
            52 & \textbf{\textcolor{OliveGreen}{69}} & 81 &
            43 & 60 & 73
        \\
     \bottomrule
     \end{tabular}
     \label{tab:megadepth_large}
\end{table*}
\begin{table*}
 \small
     \centering
     \caption{\textbf{Performance on jointly rotated images vs steerer performance}. 
     We evaluate three descriptors on image pairs where both images are rotated an equal multiple of $\ang{90}$ from upright.
     This gives an upper bound on how good the performance of test time augmentation can be.
     We compare to the performance of a steerer trained for the fixed descriptor (Setting A).
     Finally we show the performance of a descriptor DeDoDe-B$^\dag$ which is trained with data augmentation with jointly rotated images.
     For DeDoDe-B$^\dag$ we also use Setting A, so it is trained without a steerer and then a steerer is trained with the fixed descriptor.
     We use $20,000$ DeDoDe keypoints throughout.
     }
     \begin{tabular}{
        l
        rrr
        rrr
        rrr
     }
     \toprule
      & & & & \multicolumn{3}{l}{MegaDepth-1500} &  \multicolumn{3}{l}{MegaDepth-C4}  \\ 
      & \multicolumn{3}{l}{\phantom{AUC @ }MegaDepth-1500} & \multicolumn{3}{r}{joint rotation} &  \multicolumn{3}{r}{with steerer~~~~}  \\ 
      Descriptor & AUC $@$
      ~$5^{\circ}$&$10^{\circ}$&$20^{\circ}$ & 
      ~$5^{\circ}$&$10^{\circ}$&$20^{\circ}$ & 
      ~$5^{\circ}$&$10^{\circ}$&$20^{\circ}$\\
    \midrule
         DeDoDe-B \cite{edstedt2024dedode} & 
            49 & 65 & 77 &
            46 & 62 & 74 &
            43 & 60 & 73 
        \\
         DeDoDe-G \cite{edstedt2024dedode} & 
            52 & 69 & 81 & 
            45 & 61 & 74 &
            31 & 45 & 57
        \\
        DISK \cite{tyszkiewicz2020disk} &
            34 & 49 & 62 &
            29 & 45 & 58 &
            26 & 41 & 54 
        \\
        \midrule
         DeDoDe-B$^\dag$ & 
            50 & 66 & 78 &
            50 & 66 & 78 &
            50 & 66 & 78 
        \\
        \midrule
     \bottomrule
     \end{tabular}
     \label{tab:joint_rot}
\end{table*}
\begin{figure*}
    \centering
    \includegraphics[width=0.7\linewidth]{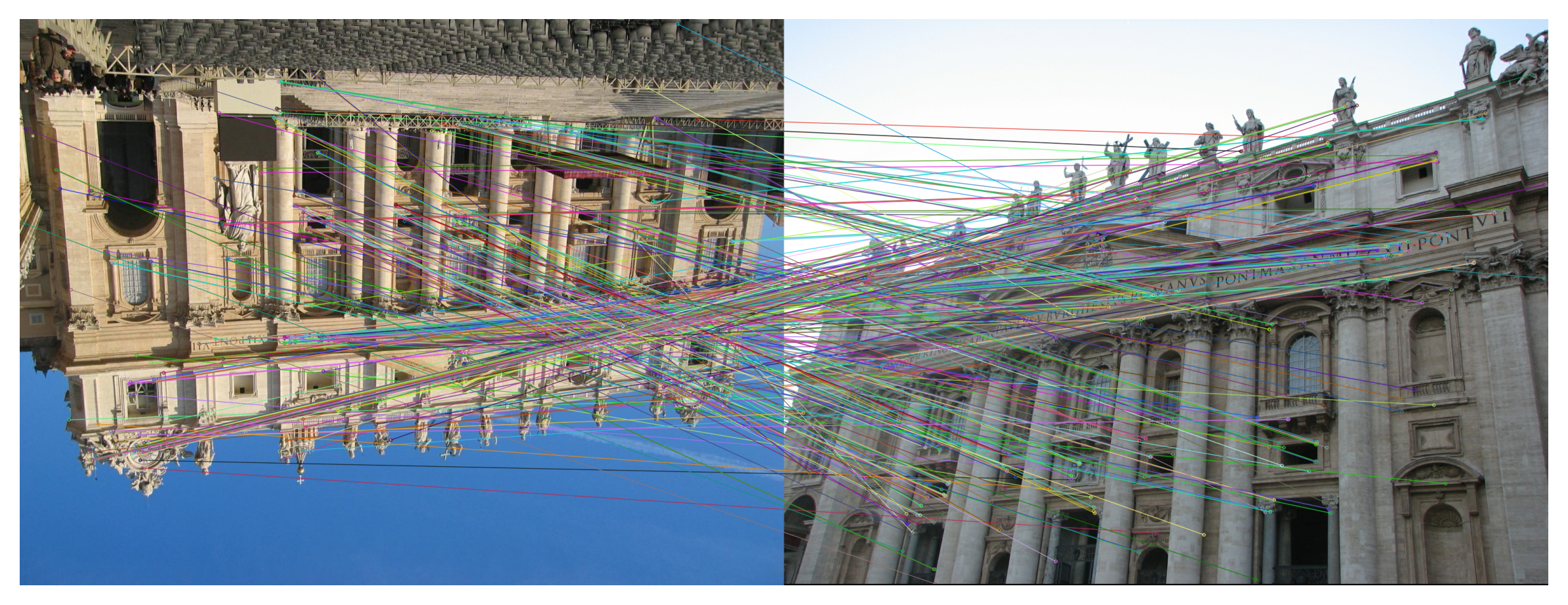}
    \\
    \includegraphics[width=0.7\linewidth]{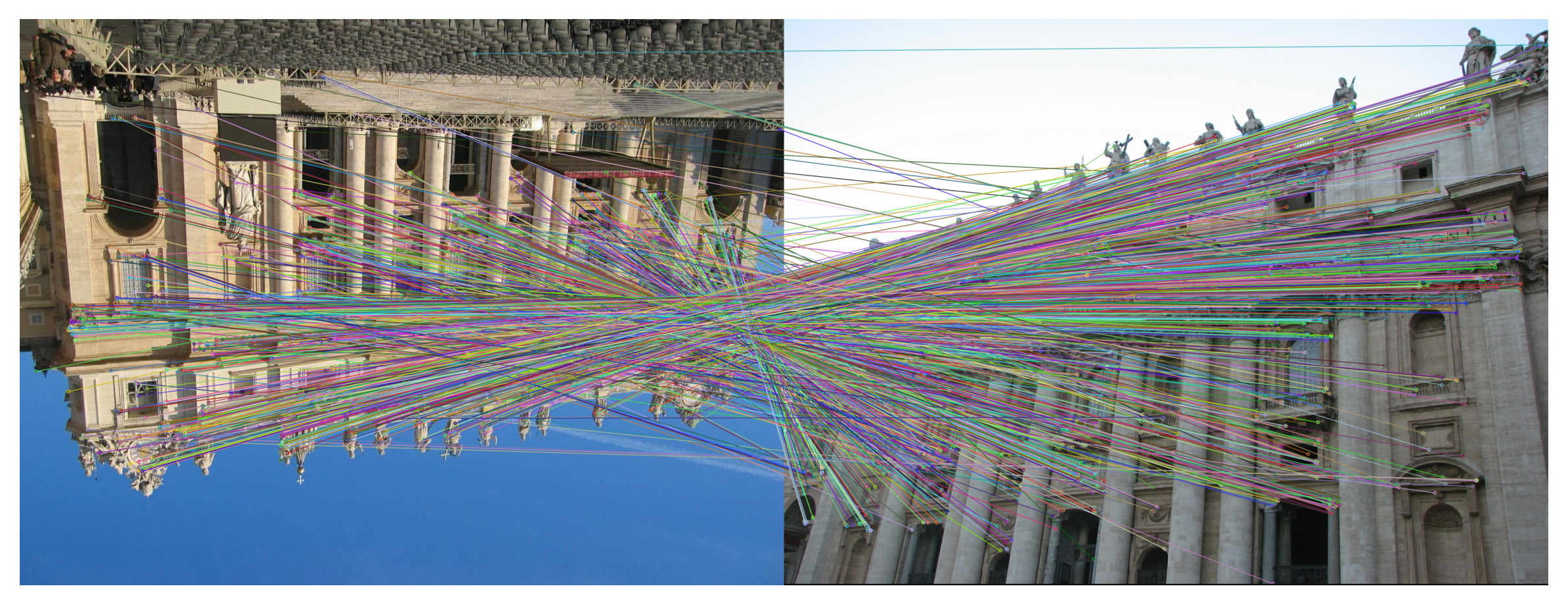}
    \caption{\textbf{Steering DeDoDe~descriptions under half turn rotations.} We replicate \cite[Figure~7]{edstedt2024dedode} but with a steerer. In the upper image pair, we match the ordinary DeDoDe-descriptions. In the lower image pair, we instead modify the descriptions of the keypoints in the right image by multiplying them by a steering matrix $\rho(\mathbf{g})^2$.
    This corresponds to setting A, where
    we, for a fixed descriptor, have optimized a steerer.
    }
    \label{fig:steer-rot2}
\end{figure*}

\end{document}